%% file: arxiv.tex
\documentclass[twoside]{article}
\usepackage[preprint]{aistats2027}
\usepackage{tabularx}

\usepackage[round]{natbib}
\renewcommand{\bibname}{References}
\renewcommand{\bibsection}{\subsubsection*{\bibname}}
\usepackage{graphicx}

\input{math_commands.tex}

\usepackage{pifont}

\usepackage{url}
\usepackage{bbm}
\usepackage{amssymb}
\usepackage{booktabs}
\usepackage{booktabs}
\usepackage{makecell}
\usepackage{threeparttable}
\usepackage{adjustbox}
\usepackage{amsthm}
\usepackage{thmtools}

\usepackage{xcolor}
\usepackage{hyperref}

\usepackage[most]{tcolorbox}
\definecolor{limitationred}{RGB}{150,40,45}
\newtcolorbox{limitationbox}{
    enhanced,
    breakable,
    colback=limitationred!4,
    colframe=limitationred!65,
    boxrule=0.6pt,
    arc=2pt,
    left=6pt,
    right=6pt,
    top=6pt,
    bottom=6pt,
    before skip=8pt,
    after skip=8pt,
    title={Limitations of single-adjustment strategies},
    fonttitle=\bfseries,
    coltitle=black
}

\def\eps{\varepsilon}

\newtheorem{proposition}{Proposition}
\newtheorem{lemma}{Lemma}

\newtheorem{stabilitycondition}{Stability Condition}
\newcounter{identificationassumption}
\renewcommand{\theidentificationassumption}{A\arabic{identificationassumption}}
\theoremstyle{remark}
\newtheorem{remark}{Remark}

\newcommand{\papertitle}{Doubly Robust Learning of Heterogeneous Survival Treatment Effects under Time-Varying Treatments}
\newcommand{\paperauthors}{%
\textbf{Tom\`as Garriga}$^{1,2,3,*}$,
\textbf{Munib Mesinovic}$^{4}$,
\textbf{Luke De Bretton-Gordon}$^{4}$,\\
\textbf{Gerard Sanz}$^{1}$,
\textbf{Axel Brando}$^{2}$,
\textbf{Tingting Zhu}$^{4}$
}

\newcommand{\paperaddresses}{%
$^1$Novartis\\
$^2$Barcelona Supercomputing Center\\
$^3$Universitat Polit\`ecnica de Catalunya\\
$^4$University of Oxford\\
$^*$Corresponding author: \texttt{tomas.garriga\_dicuzzo@novartis.com}
}

\begin{document}

\runningauthor{Garriga et al.}

\twocolumn[
\aistatstitle{\papertitle}
\aistatsauthor{\paperauthors}
\aistatsaddress{\paperaddresses}
]

\begin{abstract}
Estimating heterogeneous survival treatment effects (HSTEs) from time-to-event data under \textit{time-varying treatments} is important in fields such as medicine (e.g., to decide which sequence of treatments is better in terms of survival probability). However, this task is challenging because treatment decisions, patient characteristics, and censoring probability often evolve together, requiring \textit{sequential adjustment} for both treatments and censoring. Existing survival learners mostly consider a single baseline treatment assignment, while longitudinal treatment-effect learners typically assume uncensored outcomes. In this paper, we introduce a novel doubly robust learner for heterogeneous survival treatment effects under time-varying treatments and history-dependent censoring. Our framework separates the intervention window from the survival horizon, with treatment prescribed for a finite period, after which care follows its natural course. We further establish the favorable theoretical properties of Neyman-orthogonality, rate double robustness and quasi-oracle efficiency. The framework accommodates flexible machine learning models for nuisance estimation and treatment-effect regression. We evaluate our learners through several numerical experiments, demonstrating their effectiveness.
\end{abstract}

\section{Introduction}

In medicine, personalized treatment decisions rarely occur only once in time. In many clinical settings \citep[e.g.,][]{meneghini2007usage, press2016response}, treatments are administered repeatedly over time, while the outcome of interest is a time-to-event endpoint such as death \citep{huang2014optimization}, relapse \citep{punt2007endpoints}, or disease progression \citep{walia2023progression}. Estimating how alternative treatment sequences affect survival for a patient with a particular observed history is therefore central to individualized decision-making. In this paper, \textit{\textbf{we aim to estimate heterogeneous survival treatment effects (HSTEs) under time-varying treatments and history-dependent censoring}}.

\begin{figure*}[t]
    \centering
    \includegraphics[width=\linewidth]{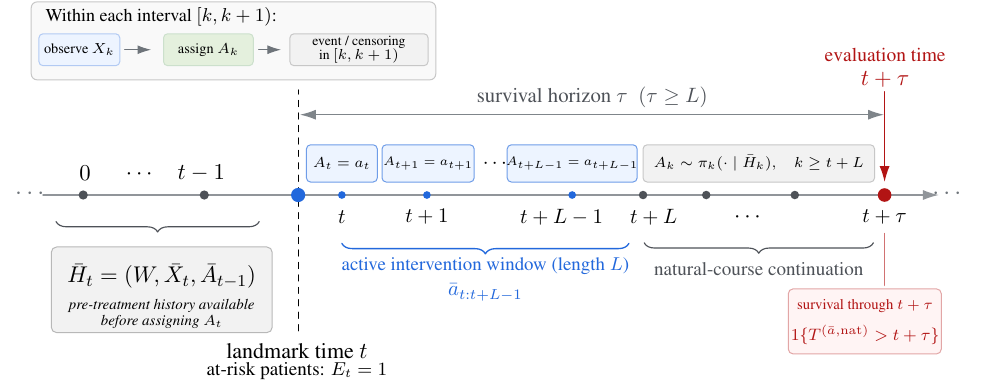}
    \caption{
        Longitudinal intervention and survival estimand. At landmark \(t\), treatment is prescribed according to \(\bar a_{t:t+L-1}\) for \(L\) decision times. Treatment subsequently follows its natural assignment mechanism, while survival is evaluated at \(t+\tau\), where \(\tau\ge L\).
    }
    \label{fig:survival-setting}
\end{figure*}

This task poses several fundamental challenges.
First, treatment decisions may depend on time-varying patient characteristics that are themselves affected by previous treatments. This produces \textit{time-varying confounding} and invalidates standard adjustment strategies that condition only on baseline information \citep{robins1986new,robins2000marginal,hernanrobins2025whatif}, requiring sequential adjustment strategies instead. Second, survival outcomes are subject to right-censoring. In a time-varying setting, the probability of remaining under observation may depend on the evolving patient history, including treatment-affected prognostic variables. Consequently, censoring may be informative given the history available at the initial treatment decision, even when it is conditionally independent of the potential event time given the evolving history. This makes it necessary to adjust for censoring \textit{sequentially} as well, making the problem more challenging than standard causal survival settings with a single baseline treatment \citep[e.g.,][]{klein2003survival, frauen2025survival}. Finally, the clinically relevant intervention and outcome windows need not coincide: a protocol may prescribe treatment for a finite period, while its effect on survival is evaluated at a later evaluation time after subsequent care has returned to its natural course.

Most existing methods for heterogeneous treatment effects estimation capture only parts of these challenges.  Causal survival learners typically handle a single baseline treatment \citep[e.g.,][]{curth2021survite,cui2023causal,frauen2025survival}, where censoring is assumed conditionally independent of the event time given a fixed baseline information set and can therefore be handled without sequential censoring adjustment. In our time-varying setting, in contrast, censoring may depend on evolving, treatment-affected histories, requiring adjustment throughout follow-up. On the other hand, longitudinal causal learners handle time-varying treatments but generally assume fully observed outcomes \citep{RMSN,gnet,melnychuk2022causal,frauen2025modelagnostic}, so they do not handle censoring. A small literature comes closer to HSTE estimation under time-varying treatments, with \citet{abraich2025tv} being our closest work. However, this work remains tied to a particular model architecture, relies on a single nuisance component and does not provide favorable learning guarantees such as double robustness.


In this paper, we develop a model-agnostic framework for learning HSTEs under time-varying treatments and history-dependent censoring. At a landmark time $t$, we compare survival under alternative treatment sequences imposed for $L$ decision times, followed by natural-course treatment, with the evaluation time $t+\tau$ potentially later (see Fig. \ref{fig:survival-setting}). This separation allows us to study finite, clinically actionable interventions. We derive complementary identification strategies based on g-computation and inverse-probability weighting.

Building on these representations, we introduce a novel learner that has several favorable learning properties: double robustness \citep{kang2007demystifying}, which allows consistency under misspecification of one of two complementary nuisance components; Neyman-orthogonality \citep{foster2023orthogonal}, which removes first-order sensitivity to nuisance-estimation error; and quasi-oracle efficiency \citep{nie2021quasi}, whereby under suitable nuisance-rate conditions the learner matches the convergence rate of its oracle counterpart. Thus, the learner combines longitudinal continuation regressions with treatment and censoring mechanisms in an orthogonal pseudo-outcome, and remains consistent if either the continuation regressions are correctly specified or both probability mechanisms are correctly specified. Its construction permits flexible machine-learning models for nuisance estimation and the final effect regression.

Our \textbf{contributions} are threefold:\footnote{Code is available at \url{https://github.com/tgarriga/SurvivalTimeDR}.}
\textbf{(1)} we formulate and identify heterogeneous survival treatment effects for finite-duration longitudinal interventions followed by natural-course treatment under history-dependent censoring, and derive several baseline learners based on complementary adjustment strategies.
\textbf{(2)} we introduce a novel doubly robust and Neyman-orthogonal learner for HSTE estimation under time-varying treatments and censoring.
\textbf{(3)} we establish its theoretical properties, including (rate) double robustness, Neyman-orthogonality and quasi-oracle efficiency, and validate the proposed learner through synthetic, semi-synthetic, and real-world experiments.

\section{Related Work\protect\footnotemark}
\footnotetext{We provide an extended related work in App. \ref{app:extended-related-work}}

\textbf{HSTE estimation with static treatments.} Several approaches have been proposed to estimate HSTEs under a single treatment assignment using baseline covariates. For example, \citet{curth2021survite} learn treatment-specific hazard functions through representation balancing, while \citet{cui2023causal} use adaptive forest neighborhoods. More recently, \citet{frauen2025survival}, \citet{xu2024estimating}, and \citet{morzywolek2023general} develop model-agnostic, doubly robust and Neyman-orthogonal survival learners. However, the previous causal methods are limited to a single treatment assignment and assume independent censoring conditional on baseline covariates. In contrast, our setting considers time-varying treatments and allows censoring to depend on the evolving patient history, requiring only sequential independent censoring. This makes the previous methods unsuitable for our task.

\textbf{HSTE estimation under time-varying treatments.} In recent years, several approaches have been proposed to estimate heterogeneous treatment effects under time-varying treatments \citep[e.g.,][]{melnychuk2022causal, gnet, frauen2025modelagnostic, garriga2026orthogen}. However, \textit{these methods do not account for censoring} and thus are not suitable for HSTE estimation. A limited literature lies at the intersection of heterogeneous effect estimation, survival outcomes, and time-varying treatments. \citet{zhu2022causal} estimate conditional survival effects from longitudinal data using recurrent neural networks. However, they explicitly assume that future covariates are independent of previous treatments, thereby excluding treatment--confounder feedback. \citet{jun2023optimizing} study heterogeneous survival effects across sequential antibiotic decisions using causal survival forests. However, they focus on an application-specific problem rather than developing a general learner for HSTE estimation under time-varying treatments. Closest to our setting, \citet{abraich2025tv} directly considers counterfactual survival under time-varying treatment sequences in the presence of treatment--confounder feedback, and targets conditional survival contrasts between alternative treatment sequences. Their approach relies on sequential representation balancing and a discriminative survival model. However, it lacks the favorable properties of double robustness and Neyman-orthogonality, and is therefore first-order sensitive to nuisance misspecification.

\textbf{Research Gap.} To the best of our knowledge, an HSTE learner that has the favorable properties of double robustness and Neyman-orthogonality while properly handling time-varying confounding and history-dependent censoring is missing.

\section{Problem Setting}

\textbf{Notation.}
Capital letters denote random variables and lowercase letters their realizations. We write \(\mathbb{P}(\cdot)\) for probabilities and \(\mathbb{E}[\cdot]\) for expectations. We use \(\widehat{\mathbb E}[Z\mid X=x]\) to denote a fitted conditional-mean regression of \(Z\) on \(X\), not an empirical average. Estimated quantities carry a hat, as in \(\widehat f\). Time is discrete and indexed by \(k\in\mathbb{N}_0\). For a process \(U\), we write \(\bar U_k:=(U_0,\ldots,U_k)\) for its trajectory up to \(k\) and \(\bar U_{r:s}:=(U_r,\ldots,U_s)\) for a contiguous segment. Throughout, \(W\) denotes static baseline covariates, \(X_k\) time-varying covariates, \(A_k\in\mathcal{A}\) the treatment assigned at decision time \(k\), and \(T,C\in\{1,2,\ldots\}\) the event and censoring times, respectively. We denote by
$\bar{H}_k:=\left(W,\bar X_k,\bar A_{k-1}\right)$
the pre-treatment history, i.e.\ all information available immediately before assigning \(A_k\), and by
$E_k:=\mathbbm{1}\{T>k,C>k\}$
the at-risk indicator. Potential quantities under a finite treatment sequence \(\bar a\) followed by natural-course treatment are written with a superscript, as in \(T^{(\bar a,\mathrm{nat})}\).

\textbf{Setup.}
We consider a discrete-time longitudinal setting with i.i.d.\ patient trajectories and no common maximum follow-up length. Within each interval, we assume a strict temporal ordering: at the beginning of interval \(k\) the covariates \(X_k\) of patients at risk are recorded, treatment \(A_k\) is assigned, and the patient is then monitored for the event or for censoring during \([k,k+1)\). Active treatment \(A_k\) is thus observed only for patients with \(E_k=1\). A patient censored exactly at time \(c\) still contributes the terminal covariates \(X_c\), and hence the successor history \(\bar{H}_c\), but neither receives \(A_c\) nor contributes any later variable; conversely, an event preempts the next covariate measurement. Writing \(\widetilde T:=\min\{T,C\}\) for the terminal follow-up time and \(\Delta:=\mathbbm{1}\{T\leq C\}\) for the event indicator, the observed data for patient \(i\) are $O_i=
\left(
W_i,
\{X_{ik}\}_{k=0}^{\widetilde T_i-\Delta_i},
\{A_{ik}\}_{k=0}^{\widetilde T_i-1},
\widetilde T_i,
\Delta_i
\right),
 i=1,\ldots,n$.

\textbf{Finite-duration interventions with natural-course continuation.}
Acute clinical protocols prescribe treatment only for a \textit{limited period}. We therefore fix a landmark time \(t\), an \textit{active intervention} length \(L\geq 1\), and a \textit{survival horizon} \(\tau\geq L\), and restrict attention to patients with \(E_t=1\). Within the follow-up window, \(\delta=0,\ldots,\tau-1\) denotes the relative time offset, corresponding to absolute time \(t+\delta\). The final prescribed decision occurs at \(t+L-1\), and the evaluation time is \(t+\tau\). A deterministic sequence
$
\bar a_{t:t+L-1}:=(a_t,\ldots,a_{t+L-1})\in\mathcal{A}^{L}
$
is applied during the \(L\) intervals \([t,t+1),\ldots,[t+L-1,t+L)\). Afterwards no further treatment is imposed and care evolves according to the natural assignment mechanism observed in the data: $\pi_k(a\mid \bar{h}_k):=\mathbb{P}(A_k=a\mid \bar{H}_k=\bar{h}_k,E_k=1), k\in\mathbb{N}_0$.
We write \(T^{(\bar a_{t:t+L-1},\mathrm{nat})}\) for the potential event time under this regime. Fig. \ref{fig:survival-setting} shows our problem setting.

\textbf{Heterogeneous survival target.}
We target survival at a prespecified survival horizon. This choice corresponds directly to clinical endpoints such as 10-day survival. Our objective is therefore a horizon-specific counterfactual survival effect. Suppressing the natural-course continuation from the superscript, our target is the conditional survival function
\begin{equation}
\begin{aligned}
S_{t,\tau}^{\bar a}(\bar{h}_t)
&:=\mathbb{P}\!\left(
T^{(\bar a_{t:t+L-1},\mathrm{nat})}>t+\tau
\mid
\right.\left.
\bar{H}_t=\bar{h}_t,E_t=1
\right),
\end{aligned}
\end{equation}
that is, the probability that a patient at risk at the landmark with history \(\bar{h}_t\) survives through \(t+\tau\) when treated according to \(\bar a_{t:t+L-1}\) and then following the natural course. For \(\tau=L\) the evaluation time coincides with the end of the intervention window, whereas \(\tau>L\) additionally captures its downstream consequences. For two candidate sequences \(\bar a_{t:t+L-1}\) and \(\bar b_{t:t+L-1}\), the heterogeneous survival treatment effect is
\begin{equation}
\label{eq:target-estimand}
\theta_{t,\tau}^{\bar a,\bar b}(\bar{h}_t):=S_{t,\tau}^{\bar a}(\bar{h}_t)-S_{t,\tau}^{\bar b}(\bar{h}_t).
\end{equation}
Throughout the methodological development, \(t\), \(L\), and \(\tau\) are fixed,
and we therefore write \(S^{\bar a}\) and \(\theta^{\bar a,\bar b}\) when no
ambiguity arises.

\textbf{Identification.}
For the treatment process, we assume the standard identification conditions for longitudinal interventions: \refstepcounter{identificationassumption}\label{ass:consistency}\textit{(\theidentificationassumption) Consistency}, if the observed treatment path over the active intervention window agrees with $\bar a_{t:t+L-1}$, then the corresponding potential event time equals the observed event time; \refstepcounter{identificationassumption}\label{ass:treatment-exchangeability}\textit{(\theidentificationassumption) Sequential treatment exchangeability}, $A_{t+\delta}\perp\!\!\!\perp T^{(\bar a_{t:t+L-1},\mathrm{nat})}\mid \bar H_{t+\delta},E_{t+\delta}=1$, for $\delta=0,\ldots,L-1$; and \refstepcounter{identificationassumption}\label{ass:treatment-positivity}\textit{(\theidentificationassumption) Treatment positivity}, $\pi_{t+\delta}(a_{t+\delta}\mid\bar H_{t+\delta})>0$ almost surely on histories compatible with the intervention, for $\delta=0,\ldots,L-1$ \citep{robins1986new,robins2000marginal,melnychuk2022causal,hernanrobins2025whatif}.

Our setting additionally requires assumptions on the censoring process. Unlike static treatment survival settings where independent censoring is assumed conditional on a fixed baseline information set, we allow censoring to depend on the evolving observed history, including treatment-affected covariates. Works like \citet{robins2000marginal,hernanrobins2025whatif,keogh2023causal} introduce similar settings. Specifically, we assume \refstepcounter{identificationassumption}\label{ass:sequential-ignorability}\textit{(\theidentificationassumption) Sequential independent censoring}: for $\delta=0,\ldots,\tau-2$, $\mathbbm{1}\{C>t+\delta+1\}
\perp\!\!\!\perp
T^{(\bar a_{t:t+L-1},\mathrm{nat})}
\mid
\bar H_{t+\delta+1},\,
T>t+\delta+1,\,
C>t+\delta$. Under our observation convention, this means that, among patients who survive through $t+\delta+1$ and whose successor history $\bar H_{t+\delta+1}$ is observed, remaining under follow-up carries no additional information about the potential event time once the up-to-date history is conditioned upon. Censoring may therefore be informative given the landmark history $\bar H_t$, while remaining conditionally independent of the potential event time given the evolving history. Finally, \refstepcounter{identificationassumption}\label{ass:sequential-positivity}\textit{(\theidentificationassumption) Censoring positivity} requires $\mathbb{P}(C>t+\delta+1\mid\bar H_{t+\delta+1},T>t+\delta+1,C>t+\delta)>0$ almost surely on the relevant histories, for $\delta=0,\ldots,\tau-2$. Treatment exchangeability and positivity are required only over the intervention window $t,\ldots,t+L-1$, whereas sequential independent censoring and censoring positivity must hold throughout the follow-up needed to determine survival at $t+\tau$.

Under the stated assumptions, our estimands $S^{\bar a}_{t,\tau}(\bar h_t)$ and $\theta^{\bar a,\bar b}_{t,\tau}(\bar h_t)$ are identified from the
observed data. We derive the corresponding identification
representations in Sec.~\ref{sec:naive_learners}.

\textbf{The challenges of our setting.} Our setting combines two specific challenges:
time-varying treatment--confounder feedback and censoring that is itself
time-varying and history-dependent.

$\bullet$ \textbf{Time-varying confounding.}
At decision time $t+\delta$, treatment may depend on the evolving history
$\bar H_{t+\delta}$, which contains prognostic covariates that can themselves
be affected by earlier treatments. These future covariates are unavailable at
the landmark, so adjustment for $\bar H_t$ alone cannot remove confounding of
subsequent treatment decisions. Conversely, conditioning on their realized
values would condition on post-treatment trajectories rather than marginalize
over the trajectories induced by the intervention. Identification therefore
requires sequential adjustment over the active intervention window
$t,\ldots,t+L-1$, for example through g-computation or treatment weighting \citep{hess2026igcnet, RMSN, frauen2025modelagnostic}.

$\bullet$ \textbf{History-dependent censoring.}
In our setting with time-varying treatments, censoring poses a qualitatively different challenge from that in standard
causal survival learning with a single baseline treatment. There, censoring
may occur throughout follow-up, but independent censoring is typically assumed
conditional on a \emph{fixed} information set, such as baseline covariates and
the baseline treatment, e.g.,
$T \perp\!\!\!\perp C \mid X,A$
\citep{frauen2025survival}.
In our longitudinal setting, by contrast, the probability of remaining
observed at time $t+\delta$ may depend on the evolving history
$\bar H_{t+\delta+1}$, including prognostic covariates affected by earlier
treatments. Hence censoring may be informative conditional on the landmark
history $\bar H_t$, even when it is conditionally independent of the potential event time given the
up-to-date history at each censoring decision \citep{keogh2023causal, robins2000marginal}. As a result, censoring cannot be adjusted for using a single baseline
censoring model. Instead, observation through $t+\tau$ is the result of a
sequence of history-dependent selection decisions, requiring stage-specific
censoring-continuation probabilities over the evolving risk sets. In
particular,
\begin{equation}
\begin{aligned}
S^{\bar a}(\bar h_t)
\neq\;
\mathbb{P}\!\Big(
&T > t+\tau
\mid
\bar H_t=\bar h_t,\,
E_t=1,\\
&\bar A_{t:t+L-1}=\bar a_{t:t+L-1},\,
C\geq t+\tau
\Big).
\end{aligned}
\end{equation}
Moreover, because the histories driving censoring may themselves be affected
by treatment, the censoring and treatment-assignment processes are coupled.
Treatment adjustment is required only over the intervention window
$t,\ldots,t+L-1$, whereas sequential censoring adjustment must continue until
survival at $t+\tau$ can be determined.

\section{Identification strategies}
\label{sec:naive_learners}

Following the potential outcomes framework \citep{hernanrobins2025whatif}, we now derive two complementary identification strategies for HSTEs under time-varying treatments, history-dependent censoring and finite-duration interventions, based on g-computation and inverse-probability weighting. These strategies yield baseline HSTE learners and form the building blocks of our doubly robust learner in Sec. \ref{sec:orthogonal_estimator}.

\subsection{Continuation-regression representation}
\label{sec:continuation_nuisances}

For a fixed intervention sequence \(\bar a\), the recursion carries information from the evaluation time back to the landmark. For \(\delta=0,\ldots,\tau-1\), let \(Y_{t+\delta+1}:=\mathbbm{1}\{T>t+\delta+1\}\) be the one-step survival indicator, and set \(m_{t+\tau}\equiv1\) as the terminal boundary condition. In full-horizon formulas, we use \(m_{t+\delta}^{\bar a}:=m_{t+\delta}\) for \(\delta\geq L\), including the boundary \(m_{t+L}^{\bar a}:=m_{t+L}\). At time \(t+\delta\), the action-value function combines survival through the next interval with the continuation survival thereafter:
\begin{equation}
\begin{aligned}
Q_{t+\delta}^{\bar a}(\bar{h}_{t+\delta},a)
&:=
\mathbb{E}\!\left[
Y_{t+\delta+1}m_{t+\delta+1}^{\bar a}(\bar{H}_{t+\delta+1})
\mid
\right.\\[-1mm]
&\qquad\left.
\bar{H}_{t+\delta}=\bar{h}_{t+\delta},
A_{t+\delta}=a,
E_{t+\delta}=1
\right],
\end{aligned}
\label{eq:action_value}
\end{equation}
When an event occurs during \([t+\delta,t+\delta+1)\), \(Y_{t+\delta+1}=0\), so the product vanishes and the unobserved successor history is irrelevant. Thus \(Q_{t+\delta}^{\bar a}\) is the probability of surviving through the evaluation time given the current history, the current treatment, and the continuation induced by \(\bar a_{t:t+L-1}\).

After the intervention ends, the successor continuation and treatment mechanism are common across regimes. Thus, at a given history and action, \(Q_{t+\delta}^{\bar a}=Q_{t+\delta}^{\bar b}=:Q_{t+\delta}\) for \(\delta=L,\ldots,\tau-1\). Averaging this action value over natural-course treatment gives the shared continuation, for \(\delta=L,\ldots,\tau-1\):
\begin{equation}
\begin{aligned}
m_{t+\delta}(\bar{h}_{t+\delta})
&:=\mathbb{E}\!\left[
Q_{t+\delta}(\bar{h}_{t+\delta},A_{t+\delta})
\right.\\[-1mm]
&\quad\left.
\mid \bar{H}_{t+\delta}=\bar{h}_{t+\delta},E_{t+\delta}=1
\right].
\end{aligned}
\label{eq:natural-continuation}
\end{equation}
During the active intervention, for \(\delta=0,\ldots,L-1\), we instead set
\begin{equation}
m_{t+\delta}^{\bar a}(\bar{h}_{t+\delta})
:=Q_{t+\delta}^{\bar a}(\bar{h}_{t+\delta},a_{t+\delta}).
\label{eq:continuation_survival}
\end{equation}
The common tail \(m_{t+\tau}\to\cdots\to m_{t+L}\) is followed by regime-specific updates through \(m_{t+L-1}^{\bar a}\to\cdots\to m_t^{\bar a}\). We call \(\{m_{t+\delta}^{\bar a}\}_{\delta=0}^{L-1}\) and \(\{m_{t+\delta}\}_{\delta=L}^{\tau-1}\) the \textit{\textbf{continuation nuisances}}; the action values \(Q_{t+\delta}^{\bar a}\) construct them. This approach adapts iterative g-computation \citep{hernanrobins2025whatif, hess2026igcnet} to our setting, and the following result proves identifiability.

\begin{proposition}
\label{prop:gcomp_identification}
Let \(\{Q_{t+\delta}^{\bar a}\}_{\delta=0}^{\tau-1}\), \(\{m_{t+\delta}^{\bar a}\}_{\delta=0}^{L-1}\), and \(\{m_{t+\delta}\}_{\delta=L}^{\tau}\) be defined by equations~(\ref{eq:action_value})--(\ref{eq:continuation_survival}) with \(m_{t+\tau}\equiv1\). Then, under (\ref{ass:consistency})--(\ref{ass:sequential-positivity}), for every landmark history \(\bar{h}_t\) with \(E_t=1\),
\begin{equation}
m_t^{\bar a}(\bar{h}_t)
=
S^{\bar a}(\bar{h}_t).
\end{equation}
\end{proposition}
\noindent\textit{Proof in App.~\ref{app:gcomp_identification}.}

In App.~\ref{app:backward-fitting} we describe a practical fitting procedure.

\subsection{Weighting representation}
\label{sec:ipw_estimator}
Inverse-probability weighting provides a complementary identification strategy by reweighting trajectories to account for both sequential treatment assignment and censoring \citep{robins2000marginal,frauen2025modelagnostic}. Treatment adjustment covers only the active intervention window \(t,\ldots,t+L-1\), while censoring adjustment continues through the follow-up needed to determine survival at \(t+\tau\). For \(\delta=0,\ldots,L-1\), define the one-step inverse-treatment factor
\begin{equation}
r_{t+\delta}^{A,\bar a}
:=\frac{\mathbbm{1}\{A_{t+\delta}=a_{t+\delta}\}}
{\pi_{t+\delta}(a_{t+\delta}\mid\bar H_{t+\delta})}.
\label{eq:treatment-one-step-factor}
\end{equation}
Its cumulative factor, for \(\delta=0,\ldots,\tau-1\), is
\begin{equation}
R_{t:t+\delta}^{A,\bar a}
:=\prod_{s=0}^{\min\{\delta,L-1\}}r_{t+s}^{A,\bar a}.
\label{eq:treatment-cumulative-factor}
\end{equation}
For \(\delta=1,\ldots,\tau-1\), retain the censoring continuation probability under our observation convention:
\begin{equation}
\begin{aligned}
G_{t+\delta}(\bar h):={}&\mathbb P\bigl(C>t+\delta\mid
\bar H_{t+\delta}=\bar h,\\
&T>t+\delta,C>t+\delta-1\bigr).
\end{aligned}
\label{eq:censoring-continuation}
\end{equation}
Define the one-step inverse-censoring factor and its cumulative product by
\begin{equation}
\begin{aligned}
r_{t+\delta}^{C}&:=G_{t+\delta}(\bar H_{t+\delta})^{-1},
&&\delta=1,\ldots,\tau-1,\\
R_{t:t+\delta}^{C}&:=\prod_{s=1}^{\delta}r_{t+s}^{C},
&&\delta=0,\ldots,\tau-1,
\end{aligned}
\label{eq:censoring-factors}
\end{equation}
with \(R_{t:t}^{C}=1\). The fitted counterparts \(\widehat r_{t+\delta}^{A,\bar a}\) and \(\widehat R_{t:t+\delta}^{A,\bar a}\) replace \(\pi\) by \(\widehat\pi\); \(\widehat r_{t+\delta}^{C}\) and \(\widehat R_{t:t+\delta}^{C}\) replace \(G\) by \(\widehat G\).

For \(\delta=0,\ldots,\tau-1\), let \(Y_{t+\delta+1}^{\mathrm{obs}}:=\mathbbm{1}\{\widetilde T>t+\delta+1\}+\mathbbm{1}\{\widetilde T=t+\delta+1,\Delta=0\}\). The resulting terminal IPW pseudo-outcome is
\begin{equation}
\Phi^{\bar a}_{\mathrm{IPW}}
:=
Y_{t+\tau}^{\mathrm{obs}}
\mathbbm{1}\{E_{t+\tau-1}=1\}
R_{t:t+\tau-1}^{A,\bar a}
R_{t:t+\tau-1}^{C}.
\label{eq:population-iptw-ipcw-pseudo-outcome}
\end{equation}
Under (\ref{ass:consistency})--(\ref{ass:sequential-positivity}), \(\mathbb{E}[\Phi^{\bar a}_{\mathrm{IPW}}\mid \bar{H}_t=\bar{h}_t,E_t=1]=S^{\bar a}(\bar{h}_t)\). Boundary conventions and nuisance estimation details are given in App.~\ref{app:ipw-details}.

\subsection{Na\"ive HSTE learners}

The previous two representations give three HSTE learners, which we call
na\"ive because each relies on only one of the two complementary nuisance
blocks and is therefore not doubly robust.

$\bullet$ \textbf{PI learner.}
The plug-in (PI) learner directly uses the continuation nuisances defined
in Sec.~\ref{sec:continuation_nuisances} to estimate HSTEs:
\begin{equation}
\widehat\theta^{\bar a,\bar b}_{\mathrm{PI}}(\bar h_t)
:=
\widehat m_t^{\bar a}(\bar h_t)
-
\widehat m_t^{\bar b}(\bar h_t).
\label{eq:gcomp-plugin-learner}
\end{equation}

$\bullet$ \textbf{RA learner.}
The regression-adjustment (RA) learner is a second-stage learner. It first
constructs an out-of-fold pseudo-outcome from the plug-in contrast:
\begin{equation}
\widehat\Phi^{\bar a,\bar b}_{\mathrm{RA}}
:=
\widehat m_t^{\bar a}(\bar H_t)
-
\widehat m_t^{\bar b}(\bar H_t).
\label{eq:gcomp-ra-pseudo-outcome}
\end{equation}
Then, using sample splitting, it fits the conditional mean of this
pseudo-outcome in a second-stage class $\mathcal F$:
$\widehat\theta^{\bar a,\bar b}_{\mathrm{RA}}(\bar h_t)
:=
\widehat{\mathbb E}\!\left[
\widehat\Phi^{\bar a,\bar b}_{\mathrm{RA}}
\mid
\bar H_t=\bar h_t,E_t=1
\right]$. This second-stage regression can be beneficial when the HSTE function is
simpler to learn than the two regime-specific conditional survival
functions separately.

$\bullet$ \textbf{IPW learner.}
The inverse-probability-weighting (IPW) learner uses the weighting
representation defined in Sec.~\ref{sec:ipw_estimator} to adjust for
treatment confounding and history-dependent censoring. Using the separate
estimated cumulative factors $\widehat R_{t:t+\delta}^{A,\bar a}$ and
$\widehat R_{t:t+\delta}^{C}$, we define
\begin{equation}
\begin{aligned}
\widehat\Phi^{\bar a}_{\mathrm{IPW}}
&:=
Y_{t+\tau}^{\mathrm{obs}}
\mathbbm{1}\{E_{t+\tau-1}=1\}
\widehat R_{t:t+\tau-1}^{A,\bar a}
\widehat R_{t:t+\tau-1}^{C},\\
\widehat\Phi^{\bar a,\bar b}_{\mathrm{IPW}}
&:=
\widehat\Phi^{\bar a}_{\mathrm{IPW}}
-
\widehat\Phi^{\bar b}_{\mathrm{IPW}}.
\end{aligned}
\label{eq:iptw-ipcw-pseudo-outcome}
\end{equation}
We then regress this pseudo-outcome on the landmark history: $
\widehat\theta^{\bar a,\bar b}_{\mathrm{IPW}}(\bar h_t)
:=
\widehat{\mathbb E}\!\left[
\widehat\Phi^{\bar a,\bar b}_{\mathrm{IPW}}
\mid
\bar H_t=\bar h_t,E_t=1
\right]$.

\textbf{Limitations of na\"ive learners.} PI and RA rely on accurate continuation nuisances, whereas IPW relies on accurate treatment and censoring mechanisms. Hence, each can be biased when its corresponding nuisance block is misspecified, and none is Neyman-orthogonal, so nuisance-estimation errors may affect the final HSTE estimate at first order. This motivates combining both representations in a doubly robust, orthogonal pseudo-outcome, as developed next.

\section{Doubly Robust Learning of HSTE\MakeLowercase{s}}
\label{sec:orthogonal_estimator}
\textbf{DR learner.} We now introduce a doubly robust (DR) and Neyman-orthogonal learner for HSTE estimation under time-varying treatments. Using the fitted continuation regressions from
Sec.~\ref{sec:continuation_nuisances} and the estimated
treatment and censoring factors from
Sec.~\ref{sec:ipw_estimator}, the DR pseudo-outcome is:
\begin{equation}
\begin{aligned}
\widehat\Phi^{\bar a}_{\mathrm{DR}}
 &:={}
\widehat m_t^{\bar a}(\bar{H}_t)+
\sum_{\delta=0}^{\tau-1}
\mathbbm{1}\{E_{t+\delta}=1\}
\\[-1mm]
&\qquad
\widehat R_{t:t+\delta}^{A,\bar a}
\widehat R_{t:t+\delta}^{C}
\Bigl[
Y_{t+\delta+1}^{\mathrm{obs}}
\widehat m_{t+\delta+1}^{\bar a}(\bar{H}_{t+\delta+1})\\[-1mm]
&\qquad\quad
-\widehat m_{t+\delta}^{\bar a}(\bar{H}_{t+\delta})
\Bigr].
\end{aligned}
\label{eq:estimated-dr-pseudo-outcome}
\end{equation}
For two regimes, we define \(\widehat\Phi^{\bar a,\bar b}_{\mathrm{DR}}:=\widehat\Phi^{\bar a}_{\mathrm{DR}}-\widehat\Phi^{\bar b}_{\mathrm{DR}}\). Then, the final HSTE learner is:
\begin{equation}
\widehat\theta^{\bar a,\bar b}_{\mathrm{DR}}(\bar h_t)
:=
\widehat{\mathbb E}\!\left[
\widehat\Phi^{\bar a,\bar b}_{\mathrm{DR}}
\mid
\bar H_t=\bar h_t,E_t=1
\right].
\end{equation}

Our DR learner estimates our target estimand in Eq. \ref{eq:target-estimand} and comes with several favorable theoretical properties that we show in the next subsection.

\subsection{Theoretical properties}

Next, we establish two fundamental properties of our DR learner: \emph{double robustness}, which implies that the DR pseudo-outcome remains conditionally
unbiased if either the continuation nuisances are correctly specified or
both the treatment and censoring mechanisms are correctly specified, and \emph{Neyman-orthogonality}, i.e., first-order insensitivity to nuisance
estimation error. 

We write
\(\Phi_{\mathrm{DR}}^{\bar a,\bar b}(O;\eta)
:=\Phi_{\mathrm{DR}}^{\bar a}(O;\eta^{\bar a})
-\Phi_{\mathrm{DR}}^{\bar b}(O;\eta^{\bar b})\), with \(\eta_0\) the true
nuisance collection and \(\widetilde\eta\) a candidate collection (App.~\ref{app:theory-preliminaries}--\ref{app:statistical-convergence}).

\begin{restatable}[Double robustness and Neyman-orthogonality]{theorem}{mainDRtheorem}
\label{thm:main-dr-orthogonality}
Under the conditions of App.~\ref{app:theory-preliminaries}, at a fixed
landmark history \(\bar h_t\),
\begin{equation}
\mathbb E\!\left[
\Phi_{\mathrm{DR}}^{\bar a,\bar b}(O;\widetilde\eta)
\mid \bar H_t=\bar h_t,E_t=1
\right]
=\theta^{\bar a,\bar b}(\bar h_t)
\label{eq:main-dr-conditional-mean}
\end{equation}
whenever, for each \(\bar d\in\{\bar a,\bar b\}\), either its complete
continuation-regression sequence is correct or both the treatment and
censoring mechanisms are correct over their respective stage ranges.
For the conditional moment
\begin{equation}
\begin{aligned}
\Psi_{\bar a,\bar b}(\eta;\bar h_t)
&:=\mathbb E\!\left[
\Phi_{\mathrm{DR}}^{\bar a,\bar b}(O;\eta)
\mid \bar H_t=\bar h_t,E_t=1
\right]\\
&\quad-\theta^{\bar a,\bar b}(\bar h_t),
\end{aligned}
\label{eq:main-dr-conditional-moment}
\end{equation}
every regular path \(\epsilon\mapsto\eta_\epsilon\) through \(\eta_0\)
for which differentiation under the conditional expectation is valid satisfies
\begin{equation}
\left.\frac{\mathrm d}{\mathrm d\epsilon}
\Psi_{\bar a,\bar b}(\eta_\epsilon;\bar h_t)
\right|_{\epsilon=0}=0.
\label{eq:main-dr-orthogonality}
\end{equation}
\end{restatable}
\noindent\textit{Proof in App.~\ref{app:dr-orthogonality}.}

The contrast is thus conditionally unbiased under these alternatives and
locally insensitive to simultaneous nuisance perturbations. We next study how these properties translate into the statistical
convergence of the estimated HSTE function. In particular, we show that
nuisance-estimation errors enter only through products of errors from
complementary nuisance components, yielding \emph{rate double robustness}
and, under sufficiently fast nuisance convergence, \emph{quasi-oracle
efficiency}.

Let \(\widehat\theta_{\mathrm{or}}^{\bar a,\bar b}\) be the output of the same
second-stage regression method as \(\widehat\theta_{\mathrm{DR}}^{\bar a,\bar b}\),
with the same function class \(\mathcal F\), loss, tuning rule, and sample splits,
but supplied with the oracle pseudo-outcome
\(\Phi_{\mathrm{DR}}^{\bar a,\bar b}(O;\eta_0)\). For a fixed \(\bar h_t\), define
\begin{equation}
\begin{aligned}
\mathcal R_{\bar h_t}(\widehat\theta)
&:=\mathbb E\!\left[
\left\{\widehat\theta(\bar h_t)
-\theta^{\bar a,\bar b}(\bar h_t)\right\}^2
\right],\\
e_{\mathrm{or}}^{\bar a,\bar b}(n)
&:=\mathcal R_{\bar h_t}
(\widehat\theta_{\mathrm{or}}^{\bar a,\bar b}).
\end{aligned}
\label{eq:main-oracle-risk}
\end{equation}
For each nuisance, let
\(e_{\pi,t+\delta}^{\bar d}(n)\), \(e_{G,t+\delta}^{\bar d}(n)\), and
\(e_{m,t+\delta}^{\bar d}(n)\) denote the squared conditional \(L_4\) error
scales defined in App.~\ref{app:statistical-convergence}, and set
\begin{equation}
\begin{aligned}
B_\eta(n)
&:=
\sum_{\bar d\in\{\bar a,\bar b\}}
\Biggl\{
\sum_{\delta=0}^{L-1}
e_{\pi,t+\delta}^{\bar d}(n)e_{m,t+\delta}^{\bar d}(n)
\\[-1mm]
&\qquad{}
+\sum_{\delta=1}^{\tau-1}
e_{G,t+\delta}^{\bar d}(n)e_{m,t+\delta}^{\bar d}(n)
\Biggr\}.
\end{aligned}
\label{eq:main-nuisance-remainder}
\end{equation}
\begin{restatable}[Rate double robustness and quasi-oracle efficiency]{theorem}{mainRateTheorem}
\label{thm:main-dr-rate}
Under the conditions of App.~\ref{app:statistical-convergence},
\begin{equation}
\mathcal R_{\bar h_t}\!\left(
\widehat\theta_{\mathrm{DR}}^{\bar a,\bar b}
\right)
\lesssim
e_{\mathrm{or}}^{\bar a,\bar b}(n)+B_\eta(n).
\label{eq:main-quasi-oracle-rate}
\end{equation}
\end{restatable}
\noindent\textit{Proof in App.~\ref{app:statistical-convergence}.}

\begin{remark}
If \(B_\eta(n)=O\!\left\{e_{\mathrm{or}}^{\bar a,\bar b}(n)\right\}\), then
\begin{equation}
\mathcal R_{\bar h_t}\!\left(
\widehat\theta_{\mathrm{DR}}^{\bar a,\bar b}
\right)
=O\!\left\{e_{\mathrm{or}}^{\bar a,\bar b}(n)\right\}.
\label{eq:main-quasi-oracle-consequence}
\end{equation}
Thus nuisance estimation does not worsen the convergence rate relative to the learner that uses the oracle nuisances. The bound also exhibits rate
double robustness: continuation-regression errors enter only in product with
treatment or censoring errors, so faster convergence of one nuisance block can
compensate for slower convergence of its complementary block.
The exact remainder is given in App.~\ref{app:exact-remainder}.
\end{remark}

\section{Evaluation}
In this section, we perform several numerical experiments with the objective of validating the effectiveness of our approach. Following standard practice in causal inference \citep{Bica2020Estimating, melnychuk2022causal}, we use synthetic and semi-synthetic datasets because they provide access to ground-truth counterfactuals. Additionally, we use a real-world dataset to demonstrate practical applicability.

\textbf{Architecture and training. } 
We use a lightweight Transformer, similar to \citet{hess2026overlapweighted}, for all nuisance and second-stage regressions.
At each time step, the longitudinal history together with the baseline covariates is
projected into a latent representation, augmented with fixed sinusoidal positional
encodings, and processed by a single causally masked Transformer encoder layer followed
by a two-layer ReLU prediction head. Continuation and second-stage models use latent
width 30, three attention heads, feed-forward width 20, and dropout $0.1$, whereas the
treatment and censoring models use width 40, four heads, and no dropout. All models are
trained with Adam ($10^{-3}$ learning rate, batch size 64) for at most 30 epochs, with
early stopping on a 15\% patient-level validation split. The hyperparameters were inherited from \citet{hess2026overlapweighted} and no additional tuning was performed. In App. \ref{app:computational-resources} we give additional information on computational resources and runtimes.

\textbf{Baselines.}
We compare our DR learner against three model-agnostic baselines derived from the complementary identification strategies in Sec.~4: the plug-in (PI) learner, which directly contrasts the estimated regime-specific survival continuations; regression adjustment (RA), which regresses the plug-in contrast on the landmark history; and IPW, which regresses the contrast of treatment- and censoring-weighted survival pseudo-outcomes. We additionally compare against TV-SurvCaus \citep{abraich2025tv}, a recent neural method for individualized survival effects under time-varying treatments, which we adapt to our finite-intervention setting and implement using the same Transformer backbone for a controlled comparison. To our knowledge, this is the only existing method for HSTE estimation under time-varying treatments.

\begin{table}[t]
\caption{\textbf{Fully synthetic dataset.}}
\centering
\scriptsize
\setlength{\tabcolsep}{2.5pt}

\resizebox{\columnwidth}{!}{%
\begin{tabular}{lcccc}
\toprule
Method
    & $L{=}2,\tau{=}2$
    & $L{=}2,\tau{=}3$
    & $L{=}2,\tau{=}4$
    & $L{=}3,\tau{=}8$ \\
\midrule
PI
    & 0.058\,{\scriptsize$\pm$0.002}
    & 0.054\,{\scriptsize$\pm$0.005}
    & 0.047\,{\scriptsize$\pm$0.011}
    & 0.047\,{\scriptsize$\pm$0.020} \\
RA
    & 0.063\,{\scriptsize$\pm$0.009}
    & 0.055\,{\scriptsize$\pm$0.008}
    & 0.048\,{\scriptsize$\pm$0.011}
    & 0.050\,{\scriptsize$\pm$0.021} \\
IPW
    & 0.081\,{\scriptsize$\pm$0.013}
    & 0.086\,{\scriptsize$\pm$0.029}
    & 0.084\,{\scriptsize$\pm$0.023}
    & 0.237\,{\scriptsize$\pm$0.073} \\
TV-SurvCaus
    & 0.058\,{\scriptsize$\pm$0.003}
    & 0.054\,{\scriptsize$\pm$0.003}
    & 0.049\,{\scriptsize$\pm$0.003}
    & 0.051\,{\scriptsize$\pm$0.005} \\
DR
    & \textbf{0.030}\,{\scriptsize\bfseries\boldmath$\pm$0.008}
    & \textbf{0.037}\,{\scriptsize\bfseries\boldmath$\pm$0.013}
    & \textbf{0.031}\,{\scriptsize\bfseries\boldmath$\pm$0.011}
    & \textbf{0.043}\,{\scriptsize\bfseries\boldmath$\pm$0.008} \\
\midrule
DR improvement (\%)
    & \textbf{48.3}
    & \textbf{31.5}
    & \textbf{34.0}
    & \textbf{8.5} \\
\bottomrule
\end{tabular}%
}

\par\vspace{1pt}
\begin{minipage}{\columnwidth}
\raggedright\tiny
Reported: RMSE of the estimated HSTE; lower is better. \par Values are mean $\pm$ 95\% $t$-interval half-width over 5 repetitions.\par
DR improvement is relative to the best competing learner.
\end{minipage}
\vspace{1pt}
\label{tab:synthetic-main}
\end{table}

\begin{table}[t]
\caption{\textbf{Semi-synthetic dataset.}}
\centering
\scriptsize
\setlength{\tabcolsep}{2.5pt}

\resizebox{\columnwidth}{!}{%
\begin{tabular}{lcccc}
\toprule
Method
    & $L{=}2,\tau{=}2$
    & $L{=}2,\tau{=}3$
    & $L{=}2,\tau{=}4$
    & $L{=}3,\tau{=}8$ \\
\midrule
PI
    & 0.073\,{\scriptsize$\pm$0.007}
    & 0.069\,{\scriptsize$\pm$0.011}
    & 0.063\,{\scriptsize$\pm$0.007}
    & 0.071\,{\scriptsize$\pm$0.014} \\
RA
    & 0.074\,{\scriptsize$\pm$0.012}
    & 0.068\,{\scriptsize$\pm$0.011}
    & 0.065\,{\scriptsize$\pm$0.007}
    & 0.071\,{\scriptsize$\pm$0.014} \\
IPW
    & 0.099\,{\scriptsize$\pm$0.007}
    & 0.116\,{\scriptsize$\pm$0.034}
    & 0.105\,{\scriptsize$\pm$0.048}
    & 0.327\,{\scriptsize$\pm$0.044} \\
TV-SurvCaus
    & 0.073\,{\scriptsize$\pm$0.001}
    & 0.071\,{\scriptsize$\pm$0.003}
    & 0.066\,{\scriptsize$\pm$0.002}
    & 0.074\,{\scriptsize$\pm$0.004} \\
DR
    & \textbf{0.046}\,{\scriptsize\bfseries\boldmath$\pm$0.008}
    & \textbf{0.048}\,{\scriptsize\bfseries\boldmath$\pm$0.008}
    & \textbf{0.047}\,{\scriptsize\bfseries\boldmath$\pm$0.006}
    & \textbf{0.055}\,{\scriptsize\bfseries\boldmath$\pm$0.011} \\
\midrule
DR improvement (\%)
    & \textbf{37.0}
    & \textbf{29.4}
    & \textbf{25.4}
    & \textbf{22.5} \\
\bottomrule
\end{tabular}%
}

\par\vspace{1pt}
\begin{minipage}{\columnwidth}
\raggedright\tiny
Reported: RMSE of the estimated HSTE; lower is better. \par Values are mean $\pm$ 95\% $t$-interval half-width over 5 repetitions.\par
DR improvement is relative to the best competing learner.
\end{minipage}

\vspace{1pt}
\label{tab:semisynthetic-main}
\end{table}

\textbf{Fully synthetic experiment.} We first consider a fully synthetic longitudinal process with treatment--confounder feedback, in which patient severity drives treatment assignment, treatment affects subsequent severity, and censoring depends on post-landmark covariates (App.~\ref{sec:dgp-synthetic}). Each repetition contains independent nuisance, second-stage, and test cohorts of 1{,}800, 1{,}800, and 1{,}200 patients, respectively, observed over 32 decision times. Patient-level ground-truth HSTEs are obtained through paired Monte Carlo counterfactual rollouts, allowing direct evaluation of individual effect estimates (App.~\ref{app:synthetic}).

\underline{Results.} The DR learner achieves the lowest RMSE in all four settings, reducing error by 9--48\% relative to the best competing method (Table~\ref{tab:synthetic-main}). Across configurations, DR consistently outperforms PI, RA, IPW, and TV-SurvCaus, showing that the doubly robust construction translates into a clear empirical advantage. In particular, while PI, RA, and TV-SurvCaus attain similar but systematically higher errors, IPW is substantially less accurate. Since DR and RA rely on the same continuation regressions and differ through the weighted orthogonal correction, their performance gap highlights the practical benefit of augmenting regression adjustment with the doubly robust orthogonal correction.
Additional diagnostics and checks using true treatment and censoring probabilities are provided in App.~\ref{app:synthetic-experiment}.

\textbf{Semi-synthetic experiment.}
We next consider a semi-synthetic benchmark based on the first 24 hours of ICU history from 6{,}000 adult MIMIC-IV patients \citep{johnson2023mimic}, using heart rate, mean arterial pressure, respiratory rate, age, and sex as the landmark history. Treatment, subsequent covariates, events, and censoring are then simulated, preserving realistic patient histories while retaining known counterfactual effects for evaluation (App.~\ref{app:semisynthetic}).

\underline{Results.}
The DR learner achieves the lowest RMSE in all four settings, reducing error by 23--37\% relative to the best competing learner (Table~\ref{tab:semisynthetic-main}). PI, RA, and TV-SurvCaus perform similarly but have consistently higher errors, while IPW is less accurate throughout.

Additional experiments further support this interpretation: targeted nuisance misspecification produces the qualitative double-robustness behavior predicted by theory (App.~\ref{app:semisynthetic-misspecification}), while stress tests show that DR remains competitive under substantially heavier censoring, sharper treatment assignment, and higher event rates, achieving the lowest RMSE in all seven short-horizon conditions and five of seven long-horizon conditions (App.~\ref{app:semisynthetic-stress}).

\textbf{Real-world experiment.}
Finally, we consider a real-world cohort of MIMIC-IV patients receiving vasopressors in the ICU. Since counterfactual outcomes are unavailable in real data, we evaluate the survival predicted under continuing vasopressors for six more hours against the observed outcomes of patients who followed that sequence, weighting each patient by the inverse of their estimated probability of continuing treatment \citep{keogh2024prediction}. We report the weighted Brier score of these survival predictions, with additional cohort and evaluation details in App.~\ref{app:realworld}.

\begin{table}[t]
\caption{\textbf{Real-world dataset.}}
\centering
\scriptsize
\setlength{\tabcolsep}{4pt}

\begin{tabular}{lcc}
\toprule
Method
    &  ($L=2,\tau=8$)
    &  ($L=2,\tau=16$) \\
\midrule
PI
    & \textbf{0.0552}\,{\tiny\bfseries\boldmath$\pm$0.0010}
    & 0.0841\,{\tiny$\pm$0.0014} \\
RA
    & 0.0598\,{\tiny$\pm$0.0025}
    & 0.0882\,{\tiny$\pm$0.0029} \\
IPW
    & 0.0589\,{\tiny$\pm$0.0013}
    & 0.0836\,{\tiny$\pm$0.0031} \\
TV-SurvCaus
    & 0.0634\,{\tiny$\pm$0.0006}
    & 0.0945\,{\tiny$\pm$0.0027} \\
DR
    & \textbf{0.0552}\,{\tiny\bfseries\boldmath$\pm$0.0020}
    & \textbf{0.0788}\,{\tiny\bfseries\boldmath$\pm$0.0015} \\
\midrule
DR improvement (\%)
    & 0.0
    & 5.7 \\
\bottomrule
\end{tabular}

\par\vspace{1pt}
\begin{minipage}{\columnwidth}
\raggedright\tiny
Reported: Weighted Brier score of predicted survival; lower is better. \par Values are mean $\pm$ 95\% $t$-interval half-width over five random seeds.\par
DR improvement is relative to the best competing learner.
\end{minipage}

\vspace{1pt}
\label{tab:realworld-main}
\end{table}

\underline{Results.}
DR ties PI for the lowest weighted Brier score at $\tau=8$ (24 hours) and achieves the lowest score at $\tau=16$ (48 hours), reducing error by 5.7\% relative to the best competing learner (Table~\ref{tab:realworld-main}).

\textbf{Conclusion.} We introduced a novel doubly robust learner for heterogeneous survival treatment effects under time-varying treatments. We established double robustness, Neyman-orthogonality and quasi-oracle efficiency. Across synthetic, semi-synthetic, and real-world experiments, the method showed strong empirical performance, supporting its use for survival-effect estimation under time-varying treatments.

\section*{Acknowledgments}
We sincerely thank Eduard Serrahima de Cambra for helpful discussions related to this work. We acknowledge Novartis for sponsoring Tomàs Garriga’s industrial PhD. We also thank the Government of Catalonia’s Industrial PhDs Plan for funding part of this research. We acknowledge Horizon Europe Programme under the AI4DEBUNK Project (\url{https://www.ai4debunk.eu}), grant agreement num. 101135757. Tingting Zhu was supported by the Royal Academy of Engineering under the Research Fellowship scheme.

\bibliographystyle{apalike}
\bibliography{aistats2027}

\clearpage
\appendix
\thispagestyle{empty}
\onecolumn
\raggedbottom
\aistatstitle{
Supplementary Materials}

\section{Extended related work}
\label{app:extended-related-work}
\textbf{Non-survival causal estimation under time-varying treatments.} Classical approaches for causal estimation under time-varying treatments include marginal structural models fitted using inverse probability of treatment weighting (IPTW) \citep{robins2000marginal} and the parametric g-formula, also known as g-computation \citep{taubman2009intervening}. More recent work has adapted these ideas to high-dimensional longitudinal settings using flexible neural architectures. \citet{RMSN} combine recurrent networks with propensity-based weighting, whereas \citet{gnet} use recurrent models to perform simulation-based g-computation. G-Transformer \citep{pmlr-v252-xiong24a} extends this line of work with Transformer-based sequence modeling, while IGC-Net \citep{hess2026igcnet} implements iterative regression-based g-computation using a Transformer architecture. Other approaches, including Counterfactual Recurrent Networks \citep{Bica2020Estimating} and Causal Transformer \citep{melnychuk2022causal}, rely on balanced representations of longitudinal patient histories to estimate counterfactual outcomes. A related line models the entire causal data-generating process with different generative models \citep{NIPS2017_94b5bde6,NEURIPS2020_0987b8b3,parafita2022estimand,sanchez2022diffusion} to produce counterfactual estimations. Relatedly, \citet{garriga2026cepae} develop an autoencoder-based approach that uses structural causal models for counterfactual inference in time series under interventions on observed events. Moving beyond estimator-specific neural architectures, \citet{frauen2025modelagnostic} propose a general family of model-agnostic meta-learners based on history adjustment, regression adjustment, propensity weighting, doubly robust correction, and inverse-variance weighting. \citet{hess2026overlapweighted} build on this framework with an orthogonal overlap-weighted estimator designed to mitigate the loss of treatment overlap that arises as treatment sequences become longer. 


\textbf{Marginal survival treatment effect estimation.}
A complementary line of work targets population-average rather than history-conditional survival effects. In static-treatment settings, several approaches estimate treatment-specific marginal survival curves while accounting for confounding and censoring \citep{hubbard2000nonparametric,zhang2012contrasting,bai2013doubly}. More recently, \citet{westling2024inference} propose a cross-fitted doubly robust estimator of treatment-specific marginal survival curves that accommodates flexible machine-learning nuisance estimators and provides pointwise and uniform inference. Related work on dynamic survival prediction uses longitudinal measurements to update individual risk predictions over time, for example through recurrent neural networks \citep{lee2020dynamic} or conditional variational models \citep{mesinovic2026dysurv}; however, these approaches are predictive rather than designed to estimate causal effects of alternative treatment strategies. For longitudinal treatments, marginal structural models provide a classical framework for estimating population-average causal effects in the presence of time-varying confounding \citep{robins2000marginal,hernan2000zidovudine}. Doubly robust extensions have also been developed in this setting, including estimators for longitudinal marginal structural models of survival processes \citep{yu2006double} and targeted estimators for intervention-specific hazard functions \citep{zheng2016doubly,petersen2014targeted}. More recently, \citet{keogh2023causal} study marginal structural models and sequential trial emulation for estimating marginal risk differences between sustained treatment strategies. Thus, estimation of marginal survival effects under both static and longitudinal treatments is relatively well established. Our focus instead is to learn survival conditional on an individual's observed history at a landmark, under finite treatment sequences followed by natural-course treatment, rather than averaging these contrasts over the target population.

\section{Backward fitting}
\label{app:backward-fitting}
Regressing terminal survival only on the landmark history and prescribed treatment sequence would ignore later, treatment-affected covariates that jointly govern survival, censoring, and subsequent treatment. These covariates cannot be included directly in a landmark-time learner: they are unavailable when the prediction is made, and conditioning on their realized values would not average over the trajectories induced by the intervention. Iterative g-computation performs this marginalization one time point at a time. The recursion runs through \(t+\tau-1\), rather than stopping at the final prescribed decision \(t+L-1\), because it must continue averaging over natural-course treatment and evolving histories until survival is evaluated at \(t+\tau\). At each step, restricting the regression to the contemporaneous risk set accounts for censoring under sequential independent censoring (\ref{ass:sequential-ignorability}) without requiring a censoring model.

\textbf{Backward fitting.} Start from \(\widehat m_{t+\tau}\equiv1\) and move backward through \(\delta=\tau-1,\ldots,0\), using the continuation estimate at \(t+\delta+1\) to form a regression target at \(t+\delta\). For each regime \(\bar a\), every patient in the risk set \(\mathcal I_{t+\delta}:=\{i:E_{i,t+\delta}=1\}\) receives the local pseudo-outcome
\begin{equation}
\widehat Z_{i,t+\delta}^{\bar a}
:=
\begin{cases}
0, & \widetilde T_i=t+\delta+1,\ \Delta_i=1, \\[1mm]
\widehat m_{t+\delta+1}^{\bar a}(\bar{H}_{i,t+\delta+1}), & \widetilde T_i>t+\delta+1
\ \text{or}
\bigl(\widetilde T_i=t+\delta+1,\ \Delta_i=0\bigr),
\end{cases}
\end{equation}
The first case is an event during \([t+\delta,t+\delta+1)\); the second includes patients censored exactly at \(t+\delta+1\), whose successor history is still recorded. For \(\delta=\tau-1,\ldots,L\), the target \(\widehat Z_{i,t+\delta}^{\bar a}=:\widehat Z_{i,t+\delta}\) has no regime dependence. Regress it once on \((\bar H_{i,t+\delta},A_{i,t+\delta})\) to obtain \(\widehat Q_{t+\delta}\). Since observed \(A_{i,t+\delta}\) follows \(\pi_{t+\delta}\) given \(\bar H_{i,t+\delta}\), regressing the generated outcome \(\widetilde M_{i,t+\delta}:=\widehat Q_{t+\delta}(\bar H_{i,t+\delta},A_{i,t+\delta})\) on \(\bar H_{i,t+\delta}\) yields the shared \(\widehat m_{t+\delta}\) in \eqref{eq:natural-continuation} without fitting a natural-course propensity. For \(\delta=L-1,\ldots,0\), use the common boundary \(\widehat m_{t+L}^{\bar a}:=\widehat m_{t+L}\), fit \(\widehat Q_{t+\delta}^{\bar a}\) separately for each regime, and set \(\widehat m_{t+\delta}^{\bar a}(\bar h_{t+\delta}):=\widehat Q_{t+\delta}^{\bar a}(\bar h_{t+\delta},a_{t+\delta})\). When \(\tau=L\), the shared tail is empty and the boundary equals one.

The common conditional tail does not imply that the post-intervention histories have the same distribution under different regimes.

\section{Proof of Proposition~\NoCaseChange{\ref{prop:gcomp_identification}}}
\label{app:gcomp_identification}

\begin{proof}
Let
\begin{equation}
Y_{t+\tau}^{\bar a}:=\mathbbm{1}\!\left\{T^{(\bar a_{t:t+L-1},\mathrm{nat})}>t+\tau\right\}.
\label{eq:gcomp-proof-potential-survival}
\end{equation}
Call a history at time \(t+\delta\) regime-compatible if its treatments at
times \(t+s\), for \(0\leq s\leq\min\{\delta-1,L-1\}\), agree with \(\bar a\);
after \(t+L\), its treatment components are left at their natural values.
On such histories, define the causal continuation quantity
\begin{equation}
V_{t+\delta}^{\bar a}(\bar{h}_{t+\delta})
:=
\mathbb{E}\!\left[Y_{t+\tau}^{\bar a}\mid
\bar{H}_{t+\delta}=\bar{h}_{t+\delta},E_{t+\delta}=1\right].
\label{eq:gcomp-proof-causal-continuation}
\end{equation}
Because prior treatment is contained in \(\bar{h}_{t+\delta}\), this is the probability
of survival through \(t+\tau\), conditional on the current at-risk history,
when any remaining treatments at offsets \(s=\delta,\ldots,L-1\) are set to \(a_{t+s}\)
and treatment from time \(t+L\) onward follows its natural course. In particular,
\(V_t^{\bar a}(\bar{h}_t)=S^{\bar a}(\bar{h}_t)\). For \(\delta\geq L\),
conditional on a given compatible history, only natural-course treatment
remains; this does not equate the distributions of histories across regimes.
We prove from the observed-data recursion that \(m_{t+\delta}^{\bar a}=V_{t+\delta}^{\bar a}\)
by backward induction, using the common-continuation convention after \(L\).

First, on a regime-compatible history with \(E_{t+\tau}=1\), consistency gives
\(Y_{t+\tau}^{\bar a}=\mathbbm{1}\{T>t+\tau\}=1\), and hence
\(V_{t+\tau}^{\bar a}=1=m_{t+\tau}\). At the first nontrivial step,
\begin{equation}
Q_{t+\tau-1}^{\bar a}(\bar{h}_{t+\tau-1},a)
=
\mathbb{E}\!\left[
Y_{t+\tau}\;\middle|\;
\substack{
\bar{H}_{t+\tau-1}=\bar{h}_{t+\tau-1},\\
A_{t+\tau-1}=a,\\
E_{t+\tau-1}=1
}
\right].
\label{eq:gcomp-proof-terminal-action}
\end{equation}
Since \(\tau\geq L\), the last pre-horizon decision time \(t+\tau-1\) is no
earlier than the final prescribed time \(t+L-1\), and consistency gives
\begin{equation}
Q_{t+\tau-1}^{\bar a}(\bar{h}_{t+\tau-1},a)
=\mathbb{E}\!\left[
Y_{t+\tau}^{\bar a}\;\middle|\;
\substack{
\bar{H}_{t+\tau-1}=\bar{h}_{t+\tau-1},\\
A_{t+\tau-1}=a,\\
E_{t+\tau-1}=1
}
\right]
\label{eq:gcomp-proof-consistent-action}
\end{equation}
for an action compatible with the regime. If \(\tau=L\), sequential treatment
exchangeability (\ref{ass:treatment-exchangeability}) therefore gives
\begin{equation}
m_{t+\tau-1}^{\bar a}(\bar{h}_{t+\tau-1})
=Q_{t+\tau-1}^{\bar a}(\bar{h}_{t+\tau-1},a_{t+L-1})
=V_{t+\tau-1}^{\bar a}(\bar{h}_{t+\tau-1}).
\label{eq:gcomp-proof-terminal-continuation}
\end{equation}
If \(\tau>L\), the current treatment is natural, and the tower property gives
the same equality after averaging
\(Q_{t+\tau-1}(\bar{h}_{t+\tau-1},A_{t+\tau-1})\) conditional on
\(\bar{H}_{t+\tau-1}=\bar{h}_{t+\tau-1},E_{t+\tau-1}=1\). Censoring exactly at \(t+\tau\)
does not obscure \(Y_{t+\tau}\): under the observation convention it records survival
through \(t+\tau\).

Now fix \(\delta\in\{0,\ldots,\tau-2\}\) and suppose
\(m_{t+\delta+1}^{\bar a}=V_{t+\delta+1}^{\bar a}\) on the relevant histories. Let
\(D_{t+\delta+1}:=\mathbbm{1}\{C>t+\delta+1\}\). Conditional on
\(E_{t+\delta}=1\) and \(Y_{t+\delta+1}=1\), the successor history
\(\bar{H}_{t+\delta+1}\) is observed whether \(D_{t+\delta+1}=1\) or the
patient is censored exactly at \(t+\delta+1\). For \(D_{t+\delta+1}=1\), the
induction hypothesis gives the conditional mean of \(Y_{t+\tau}^{\bar a}\) as
\(m_{t+\delta+1}^{\bar a}(\bar{H}_{t+\delta+1})\). Sequential independent censoring
(\ref{ass:sequential-ignorability}) implies
\begin{equation}
\begin{aligned}
&\mathbb{E}\!\left[Y_{t+\tau}^{\bar a}\mid
\bar{H}_{t+\delta+1},T>t+\delta+1,C>t+\delta,D_{t+\delta+1}=1\right] \\
&\qquad=
\mathbb{E}\!\left[Y_{t+\tau}^{\bar a}\mid
\bar{H}_{t+\delta+1},T>t+\delta+1,C>t+\delta\right],
\end{aligned}
\label{eq:gcomp-proof-censoring-transport}
\end{equation}
and censoring positivity makes the left-hand conditional mean available on all
relevant successor histories. The same continuation value therefore applies
when \(C=t+\delta+1\). If an event occurs in \([t+\delta,t+\delta+1)\),
then \(Y_{t+\delta+1}=0\); consistency and the temporal ordering give
\(Y_{t+\tau}^{\bar a}=0\) on a regime-compatible path through \(t+\delta\),
and the unavailable successor history is
immaterial. Consequently, for every current action compatible with the regime,
the induction hypothesis and iterated expectations yield
\begin{equation}
Q_{t+\delta}^{\bar a}(\bar{h}_{t+\delta},a)
=
\mathbb{E}\!\left[Y_{t+\tau}^{\bar a}\mid
\bar{H}_{t+\delta}=\bar{h}_{t+\delta},A_{t+\delta}=a,E_{t+\delta}=1\right].
\label{eq:gcomp-proof-action-value}
\end{equation}

If \(\delta<L\), treatment is prescribed. Treatment positivity makes the
conditional regression at \(a_{t+\delta}\) well defined, and sequential treatment
exchangeability (\ref{ass:treatment-exchangeability}) and consistency applied to
\eqref{eq:gcomp-proof-action-value} give
\begin{equation}
m_{t+\delta}^{\bar a}(\bar{h}_{t+\delta})
=Q_{t+\delta}^{\bar a}(\bar{h}_{t+\delta},a_{t+\delta})
=\mathbb{E}\!\left[Y_{t+\tau}^{\bar a}\mid
\bar{H}_{t+\delta}=\bar{h}_{t+\delta},E_{t+\delta}=1\right]
=V_{t+\delta}^{\bar a}(\bar{h}_{t+\delta}).
\label{eq:gcomp-proof-prescribed-continuation}
\end{equation}
If \(\delta\geq L\), no treatment is imposed at time \(t+\delta\). Rather,
\(A_{t+\delta}\mid(\bar{H}_{t+\delta}=\bar{h}_{t+\delta},E_{t+\delta}=1)\)
has its natural conditional law \(\pi_{t+\delta}(\cdot\mid \bar{h}_{t+\delta})\).
Hence, without fixing a post-intervention treatment
path,
\begin{equation}
\begin{aligned}
m_{t+\delta}(\bar{h}_{t+\delta})
&=\mathbb{E}\!\left[Q_{t+\delta}(\bar{h}_{t+\delta},A_{t+\delta})
\mid \bar{H}_{t+\delta}=\bar{h}_{t+\delta},E_{t+\delta}=1\right]\\
&=\mathbb{E}\!\left[
\mathbb{E}\!\left[Y_{t+\tau}^{\bar a}\mid
\bar{H}_{t+\delta},A_{t+\delta},E_{t+\delta}=1\right]
\mid \bar{H}_{t+\delta}=\bar{h}_{t+\delta},E_{t+\delta}=1\right]\\
&=V_{t+\delta}^{\bar a}(\bar{h}_{t+\delta}).
\end{aligned}
\label{eq:gcomp-proof-natural-continuation}
\end{equation}
This identifies the post-intervention conditional continuation as the same
function of the current history for either regime. The induction through
\(\delta=0\) gives
\begin{equation}
m_t^{\bar a}(\bar{h}_t)=V_t^{\bar a}(\bar{h}_t)=S^{\bar a}(\bar{h}_t).
\label{eq:gcomp-proof-landmark-identification}
\end{equation}
The censoring mechanism need not be modeled in this recursion: sequential
independent censoring (\ref{ass:sequential-ignorability}) is used only to transport the continuation regression
from patients with \(C>t+\delta+1\) to those censored at \(t+\delta+1\),
while the latter still contribute \(\bar{H}_{t+\delta+1}\) and an event makes
the regression target zero.
\end{proof}

\section{Inverse probability weighting details}
\label{app:ipw-details}
Inverse-probability weighting creates a pseudo-population in which the prescribed treatment sequence is followed and selective loss to follow-up is removed \citep{robins2000marginal,frauen2025modelagnostic}. The two corrections operate over different ranges: treatment is prescribed only through \(t+L-1\), whereas observation must be maintained long enough to determine survival at \(t+\tau\). For \(\delta=0,\ldots,\tau-1\), the indicator \(Y_{i,t+\delta+1}^{\mathrm{obs}}\) evaluates \(Y_{t+\delta+1}^{\mathrm{obs}}\) from Sec.~\ref{sec:ipw_estimator} on patient \(i\)'s trajectory; it includes censoring exactly at \(t+\delta+1\), when survival through that time is known. The cumulative treatment factor is \(R_{i,t:t+\delta}^{A,\bar a}=\prod_{s=0}^{\min\{\delta,L-1\}}r_{i,t+s}^{A,\bar a}\); we estimate \(\pi_{t+\delta}\) by regressing \(A_{i,t+\delta}\) on \(\bar H_{i,t+\delta}\) in the risk set for \(\delta=0,\ldots,L-1\).

In contrast to treatments, censoring reweighting continues after the intervention. The continuation probability \(G_{t+\delta}\) for \(\delta=1,\ldots,\tau-1\) has exactly the conditioning set in \eqref{eq:censoring-continuation}. At stage \(t+\delta\), the cumulative factor is \(R_{i,t:t+\delta}^{C}=\prod_{s=1}^{\delta}r_{i,t+s}^{C}\), with \(R_{i,t:t}^{C}=1\). It corrects the sequential selection required to remain observed through the beginning of interval \(t+\delta\). At the terminal score stage \(t+\tau-1\), the product stops at \(G_{t+\tau-1}\) when \(\tau>1\), rather than including \(G_{t+\tau}\): censoring exactly at the horizon still reveals survival through \(t+\tau\). We estimate \(G_{t+\delta}\) among patients known to be event-free through \(t+\delta\), using \(D_{i,t+\delta}:=\mathbbm{1}\{\widetilde T_i>t+\delta\}\) as the response and \(\bar H_{i,t+\delta}\) as predictors.

Thus, the treatment factor reconstructs adherence only while treatment is imposed, whereas the censoring factor reconstructs observation for as long as the endpoint requires. Empty products equal one throughout. Any weighted term at stage \(t+\delta\) is defined to be zero when \(E_{i,t+\delta}=0\), without evaluating unavailable factors. Since \(Y_{i,t+\tau}^{\mathrm{obs}}=1\) implies \(E_{i,t+\tau-1}=1\), the terminal pseudo-outcome can be written directly as
\begin{equation}
\Phi^{\bar a}_{\mathrm{IPW},i}
:=
Y_{i,t+\tau}^{\mathrm{obs}}
\mathbbm{1}\{E_{i,t+\tau-1}=1\}
R_{i,t:t+\tau-1}^{A,\bar a}
R_{i,t:t+\tau-1}^{C}.
\label{eq:empirical-iptw-ipcw-pseudo-outcome}
\end{equation}
Under (\ref{ass:consistency})--(\ref{ass:sequential-positivity}), its conditional expectation given \(\bar{H}_t=\bar{h}_t\) and \(E_t=1\) equals \(S^{\bar a}(\bar{h}_t)\).

\section{Theoretical properties of the DR learner}
\label{app:dr-pseudo-outcome}

\subsection{Preliminaries and nuisance notation}
\label{app:theory-preliminaries}

Fix a regime \(\bar d\in\{\bar a,\bar b\}\), and write
\(\delta=0,\ldots,L-1\) for prescribed treatment and
\(\delta=0,\ldots,\tau-1\) for the score, which evaluates survival at
\(t+\tau\). Under the observation convention specified in the setup,
censoring exactly at \(t+\tau\) still reveals survival through that time.
Consequently, when \(\tau>1\), the last required censoring nuisance is
\(G_{t+\tau-1}\), not \(G_{t+\tau}\); for \(\tau=1\) none is required.

Collect the true score-level nuisances and a generic candidate collection as
\begin{equation}
\begin{aligned}
m^{\bar d}&:=\bigl(\{m_{t+\delta}^{\bar d}\}_{\delta=0}^{L-1},
\{m_{t+\delta}\}_{\delta=L}^{\tau-1}\bigr),\\
\widetilde m^{\bar d}&:=\bigl(\{\widetilde m_{t+\delta}^{\bar d}\}_{\delta=0}^{L-1},
\{\widetilde m_{t+\delta}\}_{\delta=L}^{\tau-1}\bigr),\\
\pi&:=\{\pi_{t+\delta}\}_{\delta=0}^{L-1},
&\widetilde\pi&:=\{\widetilde\pi_{t+\delta}\}_{\delta=0}^{L-1},\\
G&:=\{G_{t+\delta}\}_{\delta=1}^{\tau-1},
&\widetilde G&:=\{\widetilde G_{t+\delta}\}_{\delta=1}^{\tau-1},
\end{aligned}
\label{eq:theory-nuisance-collections}
\end{equation}
and write \(\eta_0^{\bar d}:=(m^{\bar d},\pi,G)\) and
\(\widetilde\eta^{\bar d}:=(\widetilde m^{\bar d},\widetilde\pi,
\widetilde G)\). In full-horizon formulas, \(m_{t+\delta}^{\bar d}:=m_{t+\delta}\)
and \(\widetilde m_{t+\delta}^{\bar d}:=\widetilde m_{t+\delta}\) for
\(\delta\geq L\): these are the same functions for both regimes, not
separately fitted nuisances. The terminal values satisfy
\(\widetilde m_{t+\tau}=m_{t+\tau}\equiv1\). For candidate mechanisms, define
\begin{equation}
\begin{aligned}
\widetilde r_{t+\delta}^{A,\bar d}
&:=\frac{\mathbbm{1}\{A_{t+\delta}=d_{t+\delta}\}}
{\widetilde\pi_{t+\delta}(d_{t+\delta}\mid\bar H_{t+\delta})},
&&\delta=0,\ldots,L-1,\\
\widetilde R_{t:t+\delta}^{A,\bar d}
&:=\prod_{s=0}^{\min\{\delta,L-1\}}\widetilde r_{t+s}^{A,\bar d},
&&\delta=0,\ldots,\tau-1,\\
\widetilde r_{t+\delta}^{C}
&:=\widetilde G_{t+\delta}(\bar H_{t+\delta})^{-1},
&&\delta=1,\ldots,\tau-1,\\
\widetilde R_{t:t+\delta}^{C}
&:=\prod_{s=1}^{\delta}\widetilde r_{t+s}^{C},
&&\delta=0,\ldots,\tau-1,
\end{aligned}
\label{eq:theory-candidate-weight-factors}
\end{equation}
where \(G_{t+\delta}\) has the conditioning set in
\eqref{eq:censoring-continuation}, and \(\widetilde G_{t+\delta}\) is defined
analogously. In particular, \(\widetilde R_{t:t}^{C}=1\).
The population score evaluated at the candidate nuisances is
\begin{align}
\Phi^{\bar d}_{\mathrm{DR}}(O;\widetilde\eta^{\bar d})
&:=
\widetilde m_t^{\bar d}(\bar H_t)
\\
&\quad+\sum_{\delta=0}^{\tau-1}
\mathbbm{1}\{E_{t+\delta}=1\}
\widetilde R_{t:t+\delta}^{A,\bar d}
\widetilde R_{t:t+\delta}^{C}
\\[-1mm]
&\qquad\times\left\{
Y_{t+\delta+1}^{\mathrm{obs}}\widetilde m_{t+\delta+1}^{\bar d}(\bar H_{t+\delta+1})
-\widetilde m_{t+\delta}^{\bar d}(\bar H_{t+\delta})
\right\}.
\label{eq:population-dr-pseudo-outcome}
\end{align}
As before, a product involving an unavailable successor history is read as
zero after an event.

The remainder calculation uses the stage-specific ratios
\begin{equation}
\alpha_{t+\delta}^{\bar d}:=
\begin{cases}
\displaystyle
\frac{\pi_{t+\delta}(d_{t+\delta}\mid\bar H_{t+\delta})}
{\widetilde\pi_{t+\delta}(d_{t+\delta}\mid\bar H_{t+\delta})},
& 0\leq\delta<L,\\[2mm]
1, & L\leq\delta<\tau,
\end{cases}
\qquad
\gamma_{t+\delta}:=
\begin{cases}
1, & \delta=0,\\[1mm]
\displaystyle
\frac{G_{t+\delta}(\bar H_{t+\delta})}
{\widetilde G_{t+\delta}(\bar H_{t+\delta})}, & 1\leq\delta<\tau,
\end{cases}
\label{eq:assignment-ratios}
\end{equation}
and the pre-stage weight
\begin{equation}
\widetilde\Lambda_t^{\bar d}:=1,
\qquad
\widetilde\Lambda_{t+\delta}^{\bar d}
:=
\mathbbm{1}\{E_{t+\delta-1}=1\}Y_{t+\delta}^{\mathrm{obs}}
\widetilde R_{t:t+\delta-1}^{A,\bar d}
\widetilde R_{t:t+\delta-1}^{C},
\quad\delta=1,\ldots,\tau-1.
\label{eq:pre-stage-weight}
\end{equation}
All conditional statements below are for almost every \(\bar h_t\) in the
at-risk support. In addition to (\ref{ass:consistency})--(\ref{ass:sequential-positivity}), we require the displayed terms to
be integrable and the candidate inverse weights to be bounded. For estimated
nuisances, all population calculations are conditional on the sample used to
fit them.

\subsection{Exact conditional remainder}
\label{app:exact-remainder}

\begin{lemma}[Conditional second-order remainder]
\label{lem:conditional-second-order-remainder}
Under the conditions in Sec.~\ref{app:theory-preliminaries},
\begin{align}
&\mathbb E\!\left[
\Phi^{\bar d}_{\mathrm{DR}}(O;\widetilde\eta^{\bar d})
-\Phi^{\bar d}_{\mathrm{DR}}(O;\eta_0^{\bar d})
\mid \bar H_t=\bar h_t,E_t=1
\right]
\nonumber\\
&\quad=
\mathbb E\!\left[
\sum_{\delta=0}^{\tau-1}
\widetilde\Lambda_{t+\delta}^{\bar d}
\bigl(1-\gamma_{t+\delta}\alpha_{t+\delta}^{\bar d}\bigr)
\bigl\{\widetilde m_{t+\delta}^{\bar d}-m_{t+\delta}^{\bar d}\bigr\}(\bar H_{t+\delta})
\mid \bar H_t=\bar h_t,E_t=1
\right].
\label{eq:dr-exact-remainder}
\end{align}
Moreover,
\begin{equation}
\mathbb E\!\left[
\Phi^{\bar d}_{\mathrm{DR}}(O;\eta_0^{\bar d})
\mid\bar H_t=\bar h_t,E_t=1
\right]
=S^{\bar d}(\bar h_t).
\label{eq:true-score-targeting}
\end{equation}
Every summand in \eqref{eq:dr-exact-remainder} is a continuation-regression
error multiplied by a treatment and/or censoring error. For \(\delta\geq L\),
the continuation error is that of the shared \(m_{t+\delta}\). Explicitly,
\begin{equation}
1-\gamma_{t+\delta}\alpha_{t+\delta}^{\bar d}
=
\begin{cases}
\displaystyle
\frac{\widetilde\pi_t(d_t\mid\bar H_t)-\pi_t(d_t\mid\bar H_t)}
{\widetilde\pi_t(d_t\mid\bar H_t)}, & \delta=0,\\[4mm]
\begin{aligned}
&\frac{\widetilde\pi_{t+\delta}(d_{t+\delta}\mid\bar H_{t+\delta})
-\pi_{t+\delta}(d_{t+\delta}\mid\bar H_{t+\delta})}
{\widetilde\pi_{t+\delta}(d_{t+\delta}\mid\bar H_{t+\delta})}\\[1mm]
&\quad+\frac{\pi_{t+\delta}(d_{t+\delta}\mid\bar H_{t+\delta})}
{\widetilde\pi_{t+\delta}(d_{t+\delta}\mid\bar H_{t+\delta})}
\frac{\widetilde G_{t+\delta}(\bar H_{t+\delta})-G_{t+\delta}(\bar H_{t+\delta})}
{\widetilde G_{t+\delta}(\bar H_{t+\delta})}
\end{aligned}, & 1\leq\delta<L,\\[4mm]
\displaystyle
\frac{\widetilde G_{t+\delta}(\bar H_{t+\delta})-G_{t+\delta}(\bar H_{t+\delta})}
{\widetilde G_{t+\delta}(\bar H_{t+\delta})}, & L\leq\delta<\tau.
\end{cases}
\label{eq:assignment-error-decomposition}
\end{equation}
\end{lemma}

\begin{proof}
Condition on the nuisance-training sample, and set
\begin{equation}
\varepsilon_{t+\delta}^{\bar d}(\bar H_{t+\delta})
:=
\widetilde m_{t+\delta}^{\bar d}(\bar H_{t+\delta})
-m_{t+\delta}^{\bar d}(\bar H_{t+\delta}).
\label{eq:dr-proof-continuation-error}
\end{equation}
The recursion in equations~(\ref{eq:action_value})--(\ref{eq:continuation_survival})
implies
\begin{equation}
\mathbb E\!\left[
\mathbbm{1}\{E_{t+\delta}=1\}
\widetilde R_{t:t+\delta}^{A,\bar d}
\widetilde R_{t:t+\delta}^{C}
\left\{
Y_{t+\delta+1}^{\mathrm{obs}}m_{t+\delta+1}^{\bar d}(\bar H_{t+\delta+1})
-m_{t+\delta}^{\bar d}(\bar H_{t+\delta})
\right\}
\mid \bar H_t,E_t=1
\right]=0.
\label{eq:true-residual-zero}
\end{equation}
For \(\delta<L\), the current treatment factor restricts the residual to
\(A_{t+\delta}=d_{t+\delta}\). For \(\delta\geq L\), no current treatment factor appears and the
conditional expectation averages over the natural treatment law. Sequential
independent censoring (\ref{ass:sequential-ignorability}) and the observation convention justify the use of
\(Y_{t+\delta+1}^{\mathrm{obs}}\): survival and the successor history are observed
when censoring occurs exactly at \(t+\delta+1\), whereas the product is zero after an
event.

Subtracting the true score from the candidate score and applying
\eqref{eq:true-residual-zero} gives
\begin{align}
&\mathbb E\!\left[
\Phi^{\bar d}_{\mathrm{DR}}(O;\widetilde\eta^{\bar d})
-\Phi^{\bar d}_{\mathrm{DR}}(O;\eta_0^{\bar d})
\mid \bar H_t=\bar h_t,E_t=1
\right]
\nonumber\\
&\quad=
\mathbb E\!\left[
\varepsilon_t^{\bar d}(\bar H_t)
+\sum_{\delta=0}^{\tau-1}
\mathbbm{1}\{E_{t+\delta}=1\}
\widetilde R_{t:t+\delta}^{A,\bar d}
\widetilde R_{t:t+\delta}^{C}
\left\{
Y_{t+\delta+1}^{\mathrm{obs}}\varepsilon_{t+\delta+1}^{\bar d}(\bar H_{t+\delta+1})
-\varepsilon_{t+\delta}^{\bar d}(\bar H_{t+\delta})
\right\}
\mid \bar H_t=\bar h_t,E_t=1
\right].
\label{eq:dr-bias-before-telescoping}
\end{align}
At the landmark,
\begin{equation}
\mathbb E\!\left[
\mathbbm{1}\{E_t=1\}
\widetilde R_{t:t}^{A,\bar d}\widetilde R_{t:t}^{C}
\varepsilon_t^{\bar d}(\bar H_t)
\mid \bar H_t,E_t=1
\right]
=\alpha_t^{\bar d}\varepsilon_t^{\bar d}(\bar H_t).
\label{eq:landmark-coefficient}
\end{equation}
For \(\delta=1,\ldots,\tau-1\), the positive occurrence of
\(\varepsilon_{t+\delta}^{\bar d}\) inherited from the previous stage has
coefficient \(\widetilde\Lambda_{t+\delta}^{\bar d}\), while the negative
occurrence at the current stage satisfies
\begin{align}
&\mathbb E\!\left[
\mathbbm{1}\{E_{t+\delta}=1\}
\widetilde R_{t:t+\delta}^{A,\bar d}
\widetilde R_{t:t+\delta}^{C}
\varepsilon_{t+\delta}^{\bar d}(\bar H_{t+\delta})
\mid \bar H_t,E_t=1
\right]
\nonumber\\
&\qquad=
\mathbb E\!\left[
\widetilde\Lambda_{t+\delta}^{\bar d}
\gamma_{t+\delta}\alpha_{t+\delta}^{\bar d}
\varepsilon_{t+\delta}^{\bar d}(\bar H_{t+\delta})
\mid \bar H_t,E_t=1
\right].
\label{eq:negative-stage-coefficient}
\end{align}
Indeed, passing from the preceding observed risk set into \(E_{t+\delta}=1\)
contributes \(G_{t+\delta}/\widetilde G_{t+\delta}=\gamma_{t+\delta}\).
The current treatment contributes \(\pi_{t+\delta}/\widetilde\pi_{t+\delta}
=\alpha_{t+\delta}^{\bar d}\) only for \(\delta<L\); afterward treatment
is natural and \(\alpha_{t+\delta}^{\bar d}=1\). Collecting the two
occurrences of each \(\varepsilon_{t+\delta}^{\bar d}\) in
\eqref{eq:dr-bias-before-telescoping}, with
\(\varepsilon_{t+\tau}^{\bar d}=0\), proves \eqref{eq:dr-exact-remainder}.
Equation~\ref{eq:assignment-error-decomposition} is an algebraic expansion
of \(1-\gamma_{t+\delta}\alpha_{t+\delta}^{\bar d}\).

At the true nuisances, each augmentation term has conditional mean zero by
\eqref{eq:true-residual-zero}. Hence the conditional mean of the true score is
\(m_t^{\bar d}(\bar h_t)\), which equals
\(S^{\bar d}(\bar h_t)\) by
Proposition~\ref{prop:gcomp_identification}. This proves
\eqref{eq:true-score-targeting}.
\end{proof}

\subsection{Double robustness and Neyman-orthogonality}
\label{app:dr-orthogonality}

\mainDRtheorem*

\begin{proof}
We first establish the result for a fixed regime
\(\bar d\in\{\bar a,\bar b\}\). The true score satisfies
\begin{equation}
\mathbb E\!\left[
\Phi^{\bar d}_{\mathrm{DR}}(O;\eta_0^{\bar d})
\mid\bar H_t=\bar h_t,E_t=1
\right]
=S^{\bar d}(\bar h_t).
\end{equation}
The same equality holds with \(\eta_0^{\bar d}\) replaced by
\(\widetilde\eta^{\bar d}\) if either
\begin{enumerate}
\item[(i)] \(\widetilde m_{t+\delta}^{\bar d}=m_{t+\delta}^{\bar d}\)
for \(\delta=0,\ldots,L-1\) and \(\widetilde m_{t+\delta}=m_{t+\delta}\)
for \(\delta=L,\ldots,\tau-1\); or
\item[(ii)] \(\widetilde\pi_{t+\delta}=\pi_{t+\delta}\) for
\(\delta=0,\ldots,L-1\) and \(\widetilde G_{t+\delta}=G_{t+\delta}\)
for \(\delta=1,\ldots,\tau-1\).
\end{enumerate}
Correct targeting is \eqref{eq:true-score-targeting}. Under condition (i),
every continuation error in \eqref{eq:dr-exact-remainder} is zero. Under
condition (ii), \(\alpha_{t+\delta}^{\bar d}=\gamma_{t+\delta}=1\) over their respective index
ranges, so the same remainder is zero.

For the regime-specific conditional moment, write
\begin{equation}
\Psi_{\bar d}(\eta;\bar h_t)
:=
\mathbb E\!\left[
\Phi^{\bar d}_{\mathrm{DR}}(O;\eta)
-S^{\bar d}(\bar h_t)
\mid\bar H_t=\bar h_t,E_t=1
\right],
\label{eq:dr-conditional-moment}
\end{equation}
every regular path \(\epsilon\mapsto\eta_\epsilon^{\bar d}\) through the full
nuisance vector, with \(\eta_0^{\bar d}\) at \(\epsilon=0\) and for which
differentiation under the conditional expectation is valid, satisfies
\begin{equation}
\left.
\frac{\mathrm d}{\mathrm d\epsilon}
\Psi_{\bar d}(\eta_\epsilon^{\bar d};\bar h_t)
\right|_{\epsilon=0}=0.
\end{equation}
For a regular path through the truth, both
\(1-\gamma_{t+\delta,\epsilon}\alpha_{t+\delta,\epsilon}^{\bar d}\) and
\(m_{t+\delta,\epsilon}^{\bar d}-m_{t+\delta}^{\bar d}\) are first order in \(\epsilon\),
while \(\widetilde\Lambda_{t+\delta,\epsilon}^{\bar d}\) remains bounded locally. The
finite sum in \eqref{eq:dr-exact-remainder} is therefore \(O(\epsilon^2)\), including
when all nuisance components vary along the same path. Its first derivative at
zero vanishes. Linearity gives the assertions for the contrast. In particular,
\begin{equation}
\mathbb E\!\left[
\left\{
\Phi^{\bar a}_{\mathrm{DR}}(O;\eta_0^{\bar a})
-\Phi^{\bar b}_{\mathrm{DR}}(O;\eta_0^{\bar b})
\right\}
\mid\bar H_t=\bar h_t,E_t=1
\right]
=\theta^{\bar a,\bar b}(\bar h_t).
\label{eq:dr-contrast-targeting}
\end{equation}
\end{proof}

\subsection{Statistical convergence}
\label{app:statistical-convergence}

For the two regimes, let
\begin{equation}
\eta_0
:=
\left(
\{\{m_{t+\delta}^{\bar d}\}_{\delta=0}^{L-1}:\bar d\in\{\bar a,\bar b\}\},
\{m_{t+\delta}\}_{\delta=L}^{\tau-1},
\{\pi_{t+\delta}\}_{\delta=0}^{L-1},
\{G_{t+\delta}\}_{\delta=1}^{\tau-1}
\right)
\label{eq:contrast-nuisance-collection}
\end{equation}
collect the true score-level nuisances, and, for a collection \(\eta\) of the
same form, define
\begin{equation}
\Phi^{\bar a,\bar b}_{\mathrm{DR}}(O;\eta)
:=
\Phi^{\bar a}_{\mathrm{DR}}(O;\eta^{\bar a})
-\Phi^{\bar b}_{\mathrm{DR}}(O;\eta^{\bar b}).
\label{eq:dr-contrast-pseudo-outcome}
\end{equation}
The oracle pseudo-outcome is
\(\Phi^{\bar a,\bar b}_{\mathrm{DR}}(O;\eta_0)\). Let
\(\widehat\theta_{\mathrm{or}}^{\bar a,\bar b}\) be the output of exactly the same
second-stage regression procedure used for
\(\widehat\theta_{\mathrm{DR}}^{\bar a,\bar b}\), with the same function class
\(\mathcal F\), loss, tuning rule, and sample splits, but with the oracle
pseudo-outcome in place of its cross-fitted estimate. For a fixed landmark
history, define the pointwise mean-squared risk
\begin{equation}
\mathcal R_{\bar h_t}(\widehat\theta)
:=
\mathbb E\!\left[
\left\{\widehat\theta(\bar h_t)
-\theta^{\bar a,\bar b}(\bar h_t)\right\}^2
\right],
\qquad
e_{\mathrm{or}}^{\bar a,\bar b}(n)
:=\mathcal R_{\bar h_t}
(\widehat\theta_{\mathrm{or}}^{\bar a,\bar b}).
\label{eq:oracle-risk-definition}
\end{equation}

\paragraph{Cross-fitting, overlap, and boundedness.}
Use a fixed number \(V\) of folds. If fold \(v\) contains \(n_v\) observations,
write \(\rho_v:=n_v/n\), assume the fold proportions are bounded away from
zero, and note that \(\sum_{v=1}^V\rho_v=1\). For held-out fold \(v\), write
\(\mathcal T_v\) for its complementary nuisance-training sample and attach
\((-v)\) to the resulting estimates. Each estimated pseudo-outcome is evaluated
only on its held-out fold, and the same fold partition is used at every
backward stage. We take \(L\) and \(\tau\) to be fixed. Strengthening (\ref{ass:treatment-positivity}) and (\ref{ass:sequential-positivity}) to
uniform true overlap for the rate analysis, assume that, on the relevant
regime-specific supports, for some \(c_0>0\),
\begin{equation}
\pi_{t+\delta}(d_{t+\delta}\mid\bar H_{t+\delta})\geq c_0,
\quad\delta=0,\ldots,L-1,\quad\bar d\in\{\bar a,\bar b\},
\qquad
G_{t+\delta}(\bar H_{t+\delta})\geq c_0,
\quad\delta=1,\ldots,\tau-1.
\label{eq:uniform-true-overlap}
\end{equation}
This is distinct from bounded fitted inverse weights. In implementation, the
fitted denominators entering the score are deterministically lower-clipped at
a fixed \(c_f\in(0,c_0)\):
\begin{equation}
\widehat\pi_{t+\delta}^{(-v)}(d_{t+\delta}\mid\bar H_{t+\delta})\geq c_f,
\quad\delta=0,\ldots,L-1,
\qquad
\widehat G_{t+\delta}^{(-v)}(\bar H_{t+\delta})\geq c_f,
\quad\delta=1,\ldots,\tau-1.
\label{eq:fitted-weight-clipping}
\end{equation}
Thus fitted overlap is enforced rather than assumed as a stochastic event; all
errors below refer to these clipped fits. A lower clip is sufficient because
these fitted probabilities enter the proof only through inverse weights; no
upper truncation is needed. Assume also
\(\lvert\widehat m_{t+\delta}^{\bar d,(-v)}\rvert\leq B<\infty\) for
\(\delta=0,\ldots,\tau-1\), with the shared fitted continuation when
\(\delta\geq L\). All rate statements
are as \(n\to\infty\).

\paragraph{Separate fourth-moment nuisance rates.}
For \(\bar d\in\{\bar a,\bar b\}\), let \(\mathbb E^{\bar d}\) denote
expectation under the corresponding uncensored history law: treatment follows
regime \(\bar d\) through \(\delta=L-1\), treatment is natural afterward, and censoring
is eliminated through the follow-up horizon while the event process remains
natural. This is the regime-specific law identified by the true treatment and
censoring likelihood ratios in the score. For a fresh trajectory from this law
and a fold-specific random function \(f_{t+\delta,v}\), define the joint
training-sample/evaluation-history norm
\begin{equation}
\lVert f_{t+\delta,v}\rVert_{4,\bar d,t}
:=
\left(
\mathbb E_{\mathcal T_v}\mathbb E^{\bar d}\!\left[
\lvert f_{t+\delta,v}(\bar H_{t+\delta})\rvert^4
\mid\bar H_t=\bar h_t,E_t=1
\right]
\right)^{1/4}.
\label{eq:conditional-joint-l4-norm}
\end{equation}
The fresh evaluation trajectory in \(\mathbb E^{\bar d}\) is independent of
\(\mathcal T_v\), as ensured foldwise by sample splitting. Define
\begin{align}
e_{\pi,t+\delta}^{\bar d}(n)
&:=\max_{1\leq v\leq V}
\left\lVert
\widehat\pi_{t+\delta}^{(-v)}(d_{t+\delta}\mid\cdot)
-\pi_{t+\delta}(d_{t+\delta}\mid\cdot)
\right\rVert_{4,\bar d,t}^{2},
&&\delta=0,\ldots,L-1,
\nonumber\\
e_{G,t+\delta}^{\bar d}(n)
&:=\max_{1\leq v\leq V}
\left\lVert
\widehat G_{t+\delta}^{(-v)}-G_{t+\delta}
\right\rVert_{4,\bar d,t}^{2},
&&\delta=1,\ldots,\tau-1,
\nonumber\\
e_{m,t+\delta}^{\bar d}(n)
&:=\max_{1\leq v\leq V}
\left\lVert
\widehat m_{t+\delta}^{\bar d,(-v)}-m_{t+\delta}^{\bar d}
\right\rVert_{4,\bar d,t}^{2},
&&\delta=0,\ldots,\tau-1.
\label{eq:nuisance-rate-definition}
\end{align}
For \(\delta\geq L\), \(\widehat m_{t+\delta}^{\bar d,(-v)}\) abbreviates the
same shared fit \(\widehat m_{t+\delta}^{(-v)}\) for both regimes. Its error
norm can still depend on \(\bar d\), because the evaluation-history laws differ.
These are squared conditional \(L_4\) error scales: each is the maximum over
folds of the square root of the corresponding conditional fourth moment in
\eqref{eq:conditional-joint-l4-norm}. They are not mean-squared errors. Thus,
if two nuisance \(L_4\) norm errors have orders \(a_n\) and \(b_n\), their
corresponding \(e\)-quantities have orders \(a_n^2\) and \(b_n^2\), and their
product has order \((a_nb_n)^2\), the scale of a squared second-order bias. We
assume these quantities are finite and tend to zero. Taking a maximum is
harmless for fixed \(V\) and makes the foldwise requirement explicit. No mixed
product-moment rate is assumed.

For each fold and regime, define the conditional score perturbation
\begin{equation}
\widehat b_v^{\bar d}(\bar h_t)
:=
\mathbb E\!\left[
\Phi^{\bar d}_{\mathrm{DR}}(O;\widehat\eta^{\bar d,(-v)})
-\Phi^{\bar d}_{\mathrm{DR}}(O;\eta_0^{\bar d})
\mid\bar H_t=\bar h_t,E_t=1,\mathcal T_v
\right].
\label{eq:estimated-score-bias}
\end{equation}

\begin{stabilitycondition}[Generated-outcome regression stability]
\label{ass:generated-outcome-stability}
The second-stage regression procedure satisfies the Kennedy-style stability
bound
\begin{align}
\mathcal R_{\bar h_t}
(\widehat\theta^{\bar a,\bar b}_{\mathrm{DR}})
&\lesssim
e_{\mathrm{or}}^{\bar a,\bar b}(n)
+\sum_{v=1}^{V}\rho_v
\mathbb E\!\left[
\left\{\widehat b_v^{\bar a}(\bar h_t)
-\widehat b_v^{\bar b}(\bar h_t)\right\}^2
\right].
\label{eq:generated-outcome-stability}
\end{align}
\end{stabilitycondition}
This high-level condition controls the effect of replacing the oracle
regression targets by cross-fitted generated targets
\citep{kennedy2023towards}. The oracle and feasible learners use the same
regression method, function class, loss, tuning rule, and splits. The condition
is conceptually separate from Neyman-orthogonality and is not the source of the
nuisance-product structure. 

For brevity, write
\begin{equation}
B_\eta(n)
:=
\sum_{\bar d\in\{\bar a,\bar b\}}
\left\{
\sum_{\delta=0}^{L-1}e_{\pi,t+\delta}^{\bar d}(n)e_{m,t+\delta}^{\bar d}(n)
+\sum_{\delta=1}^{\tau-1}e_{G,t+\delta}^{\bar d}(n)e_{m,t+\delta}^{\bar d}(n)
\right\}.
\label{eq:nuisance-product-remainder}
\end{equation}

\mainRateTheorem*

\begin{proof}
Under the conditions stated above and
Stability Condition~\ref{ass:generated-outcome-stability}, with fixed \(L\) and
\(\tau\),
\begin{align}
\mathcal R_{\bar h_t}
(\widehat\theta^{\bar a,\bar b}_{\mathrm{DR}})
&\lesssim
e_{\mathrm{or}}^{\bar a,\bar b}(n)+B_\eta(n)
\nonumber\\
&=
e_{\mathrm{or}}^{\bar a,\bar b}(n)
+\sum_{\bar d\in\{\bar a,\bar b\}}
\left\{
\sum_{\delta=0}^{L-1}e_{\pi,t+\delta}^{\bar d}(n)e_{m,t+\delta}^{\bar d}(n)
+\sum_{\delta=1}^{\tau-1}e_{G,t+\delta}^{\bar d}(n)e_{m,t+\delta}^{\bar d}(n)
\right\}.
\label{eq:dr-rate-direct}
\end{align}
\addtocounter{equation}{2}
Fix a held-out fold \(v\) and condition on its nuisance-training sample
\(\mathcal T_v\). Apply
Lemma~\ref{lem:conditional-second-order-remainder} with
\(\widetilde\eta^{\bar d}=\widehat\eta^{\bar d,(-v)}\). Uniform true overlap,
deterministic fitted-denominator clipping, and the fixed horizons bound all
pre-stage weights and likelihood ratios. Equation
\ref{eq:assignment-error-decomposition} expresses every summand as a
continuation error times a treatment error, a censoring error, or their sum.
Jensen's inequality, the squared triangle inequality, and the fixed number of
stages therefore give, after averaging over \(\mathcal T_v\),
\begin{align}
\mathbb E\!\left[\{\widehat b_v^{\bar d}(\bar h_t)\}^2\right]
&\lesssim
\sum_{\delta=0}^{L-1}
\mathbb E_{\mathcal T_v}\mathbb E^{\bar d}\!\left[
\left\{\widehat\pi_{t+\delta}^{(-v)}(d_{t+\delta}\mid\bar H_{t+\delta})
-\pi_{t+\delta}(d_{t+\delta}\mid\bar H_{t+\delta})\right\}^2
\right.
\nonumber\\[-1mm]
&\hspace{18mm}\left.
\left\{\widehat m_{t+\delta}^{\bar d,(-v)}(\bar H_{t+\delta})
-m_{t+\delta}^{\bar d}(\bar H_{t+\delta})\right\}^2
\mid\bar H_t=\bar h_t,E_t=1
\right]
\nonumber\\
&\quad+
\sum_{\delta=1}^{\tau-1}
\mathbb E_{\mathcal T_v}\mathbb E^{\bar d}\!\left[
\left\{\widehat G_{t+\delta}^{(-v)}(\bar H_{t+\delta})
-G_{t+\delta}(\bar H_{t+\delta})\right\}^2
\right.
\nonumber\\[-1mm]
&\hspace{18mm}\left.
\left\{\widehat m_{t+\delta}^{\bar d,(-v)}(\bar H_{t+\delta})
-m_{t+\delta}^{\bar d}(\bar H_{t+\delta})\right\}^2
\mid\bar H_t=\bar h_t,E_t=1
\right].
\label{eq:dr-bias-product-bound}
\end{align}
The change from the observed weighted histories in the exact remainder to
\(\mathbb E^{\bar d}\) changes only the implicit constant: the corresponding
Radon--Nikodym factors are uniformly bounded under
equations~(\ref{eq:uniform-true-overlap})--(\ref{eq:fitted-weight-clipping}), and the
number of stages is fixed.

Apply H\"older's inequality to each mixed term in
\eqref{eq:dr-bias-product-bound}. For example, with the expectation taken over
the joint law of \(\mathcal T_v\) and the regime-specific evaluation history,
\begin{align}
&\mathbb E_{\mathcal T_v}\mathbb E^{\bar d}\!\left[
\left\{\widehat\pi_{t+\delta}^{(-v)}(d_{t+\delta}\mid\bar H_{t+\delta})
-\pi_{t+\delta}(d_{t+\delta}\mid\bar H_{t+\delta})\right\}^2
\left\{\widehat m_{t+\delta}^{\bar d,(-v)}(\bar H_{t+\delta})
-m_{t+\delta}^{\bar d}(\bar H_{t+\delta})\right\}^2
\mid\bar H_t=\bar h_t,E_t=1
\right]
\nonumber\\
&\quad\leq
\left(
\mathbb E_{\mathcal T_v}\mathbb E^{\bar d}\!\left[
\left\lvert\widehat\pi_{t+\delta}^{(-v)}(d_{t+\delta}\mid\bar H_{t+\delta})
-\pi_{t+\delta}(d_{t+\delta}\mid\bar H_{t+\delta})\right\rvert^4
\mid\bar H_t=\bar h_t,E_t=1
\right]
\right)^{1/2}
\nonumber\\
&\qquad\times
\left(
\mathbb E_{\mathcal T_v}\mathbb E^{\bar d}\!\left[
\left\lvert\widehat m_{t+\delta}^{\bar d,(-v)}(\bar H_{t+\delta})
-m_{t+\delta}^{\bar d}(\bar H_{t+\delta})\right\rvert^4
\mid\bar H_t=\bar h_t,E_t=1
\right]
\right)^{1/2}
\nonumber\\
&\quad\leq e_{\pi,t+\delta}^{\bar d}(n)e_{m,t+\delta}^{\bar d}(n).
\label{eq:holder-derived-treatment-product}
\end{align}
The identical calculation with \(\widehat G_{t+\delta}^{(-v)}-G_{t+\delta}\) gives
\begin{align}
&\mathbb E_{\mathcal T_v}\mathbb E^{\bar d}\!\left[
\{\widehat G_{t+\delta}^{(-v)}(\bar H_{t+\delta})
-G_{t+\delta}(\bar H_{t+\delta})\}^2
\right.
\nonumber\\[-1mm]
&\hspace{18mm}\left.
\{\widehat m_{t+\delta}^{\bar d,(-v)}(\bar H_{t+\delta})
-m_{t+\delta}^{\bar d}(\bar H_{t+\delta})\}^2
\mid\bar H_t=\bar h_t,E_t=1
\right]
\nonumber\\
&\quad\leq e_{G,t+\delta}^{\bar d}(n)e_{m,t+\delta}^{\bar d}(n).
\label{eq:holder-derived-censoring-product}
\end{align}
Consequently, uniformly over the fixed folds,
\begin{equation}
\mathbb E\!\left[\{\widehat b_v^{\bar d}(\bar h_t)\}^2\right]
\lesssim
\sum_{\delta=0}^{L-1}e_{\pi,t+\delta}^{\bar d}(n)e_{m,t+\delta}^{\bar d}(n)
+\sum_{\delta=1}^{\tau-1}e_{G,t+\delta}^{\bar d}(n)e_{m,t+\delta}^{\bar d}(n).
\label{eq:single-regime-dr-rate}
\end{equation}
For each fold,
\((\widehat b_v^{\bar a}-\widehat b_v^{\bar b})^2
\leq2(\widehat b_v^{\bar a})^2+2(\widehat b_v^{\bar b})^2\).
Averaging \eqref{eq:single-regime-dr-rate} with the fold proportions
\(\{\rho_v\}_{v=1}^V\) and over the two regimes, then applying
Stability Condition~\ref{ass:generated-outcome-stability}, proves
\eqref{eq:dr-rate-direct}. Equation~\ref{eq:quasi-oracle-consequence}
follows immediately from \eqref{eq:quasi-oracle-condition}.

Cross-fitting is used here only to condition on \(\mathcal T_v\) while
analyzing its held-out fold. It does not make treatment, censoring, and
continuation errors evaluated at the same history independent; no such
independence is used in \eqref{eq:holder-derived-treatment-product} or
\eqref{eq:holder-derived-censoring-product}.
\end{proof}

\paragraph{Quasi-oracle consequence.}
In particular, if
\begin{equation}
B_\eta(n)=O\!\left\{e_{\mathrm{or}}^{\bar a,\bar b}(n)\right\},
\tag{67}
\label{eq:quasi-oracle-condition}
\end{equation}
then
\begin{equation}
\mathbb E\!\left[
\left\{\widehat\theta_{\mathrm{DR}}^{\bar a,\bar b}(\bar h_t)
-\theta^{\bar a,\bar b}(\bar h_t)\right\}^2
\right]
=O\!\left\{e_{\mathrm{or}}^{\bar a,\bar b}(n)\right\}.
\tag{68}
\label{eq:quasi-oracle-consequence}
\end{equation}
Thus the feasible DR learner has the same convergence rate as the
otherwise identical second-stage learner supplied with the oracle
pseudo-outcome, or equivalently with the true score nuisances.

The bound exhibits rate double robustness: continuation-regression error enters
only in product with treatment error for \(\delta=0,\ldots,L-1\) and with censoring
error for \(\delta=1,\ldots,\tau-1\). Hence faster convergence of either nuisance
block can compensate for slower convergence of its complementary block. As the
proof shows, this product structure is derived from the exact orthogonal
remainder, separate fourth-moment rates, and H\"older's inequality; it is not a
mixed-rate assumption. This is an oracle-rate comparison, not a claim of
semiparametric efficiency or asymptotic normality.

\section{Fully synthetic experiment}
\label{app:synthetic}
\label{sec:dgp-synthetic}

\subsection{Data-generating process}

We simulate i.i.d.\ patient trajectories over $K=32$ decision times, with the notation and
the within-interval ordering of the problem setting. The static covariates are
$W=(W_1,W_2,W_3)$ and the time-varying covariates $X_k=(S_k,R_k,B_k)$, which we interpret
as disease severity, physiological reserve, and accumulated treatment burden; treatment is
binary, $\mathcal{A}=\{0,1\}$. We write $\mathrm{expit}(x)=(1+e^{-x})^{-1}$. The mechanisms
below close the treatment--confounder feedback loop: severity raises the probability of
treatment, treatment lowers severity, and severity drives the event hazard.

\textbf{Covariates.} We draw $W_1\sim\mathcal{N}(0,1)$, $W_2\sim\mathrm{Bernoulli}(0.5)$,
and $W_3\sim\mathcal{N}(0,1)$, and initialize
\begin{align}
S_0 &= 0.55\,W_1 + 0.35\,W_2 + \eps^S_0, & \eps^S_0 &\sim \mathcal{N}(0, 0.45^2), \\
R_0 &= 0.45\,W_3 - 0.20\,W_1 + \eps^R_0, & \eps^R_0 &\sim \mathcal{N}(0, 0.40^2), \\
B_0 &= \eps^B_0, & \eps^B_0 &\sim \mathcal{N}(0, 0.15^2).
\end{align}
For patients who survive interval $k$ we update the covariates, writing $\nu=\kappa\phi$
for the feedback strength $\phi=0.8$,
\begin{align}
S_{k+1} &= 0.50\,S_k + 0.18\,S_{\max(0,k-1)} - 0.10\,R_k + 0.12\,\bar S^{(4)}_k
  + 0.10\,W_1 + 0.06\sin(k/3.5) \nonumber \\
&\quad + \nu\left[-A_k\left(0.34 + 0.12\,W_2 + 0.10\tanh\!\left(\bar S^{(4)}_k\right)\right)
  + 0.05\,A_{k-1}\right] + \eps^S_{k+1}, \\
R_{k+1} &= 0.56\,R_k + 0.16\,R_{\max(0,k-2)} - 0.10\,S_k - 0.07\,\bar S^{(4)}_k \nonumber \\
&\quad + 0.10\,W_3 - 0.05\,B_k - 0.08\,\nu\,A_k + \eps^R_{k+1}, \\
B_{k+1} &= 0.74\,B_k + \nu\left(0.65\,A_k + 0.15\,\bar A^{(4)}_k\right)
  + 0.04\tanh\!\left(\bar S^{(8)}_k\right) + \eps^B_{k+1},
\end{align}
with independent innovations $\eps^S_{k+1}\sim\mathcal{N}(0,0.28^2)$,
$\eps^R_{k+1}\sim\mathcal{N}(0,0.22^2)$, and $\eps^B_{k+1}\sim\mathcal{N}(0,0.08^2)$.

\textbf{History summaries.} The mechanisms depend on $\bar{H}_k$ through summaries that are
deliberately not first-order Markov in $X_k$, so that no learner can succeed by tracking
the current covariates alone. For a process $U$ we write, with indices truncated below at
zero,
\begin{equation}
\bar U^{(m)}_k = \frac{1}{|\mathcal{I}_m(k)|}\sum_{j \in \mathcal{I}_m(k)} U_j,
\quad \mathcal{I}_m(k) = \{\max(0, k-m+1), \dots, k\},
\qquad
\Delta U^{(3)}_k = U_k - U_{\max(0,\,k-3)},
\end{equation}
and for treatment we set $A_{-1}=0$ and use the recent and cumulative exposures
$\bar A^{(4)}_k$, the mean of $A_{\max(0,k-4)},\dots,A_{k-1}$, and
$\bar A_k=k^{-1}\sum_{j<k}A_j$, both zero at $k=0$. These summaries are internal to the
simulator: the learners receive only $\bar{H}_k$.

\textbf{Treatment, event, and censoring mechanisms.} Treatment is assigned by a logistic
model in which severity, its recent level, and its trend all raise the probability of
treatment, and past treatment is strongly persistent,
\begin{align}
\eta_k &= -0.10 + 0.62\,S_k - 0.25\,R_k - 0.12\,B_k
  + 0.28\,\bar S^{(4)}_k - 0.12\,\bar R^{(8)}_k + 0.18\,\Delta S^{(3)}_k \nonumber \\
  &\quad + 0.28\,W_2 + 0.45\,A_{k-1} + 0.22\,\bar A^{(4)}_k + 0.10\,\bar A_k
  + 0.12\sin(k/4), \\
\pi_k(1 \mid \bar{H}_k) &= \mathrm{clip}\!\left(\mathrm{expit}(\alpha\,\eta_k),\; c_\pi,\; 1 - c_\pi\right),
\end{align}
so that sequential positivity holds by construction with $\pi_k\in[c_\pi,1-c_\pi]$. The
effect of treatment varies with both baseline and time-varying history,
\begin{equation}
\begin{split}
\theta(\bar{H}_k) = \gamma\Big[&0.48 + 0.16\,W_2 + 0.18\tanh(S_k) + 0.14\tanh\!\left(\bar S^{(4)}_k\right) \\
  &- 0.10\,B_k - 0.08\,\bar A^{(4)}_k + 0.06\tanh\!\left(\Delta S^{(3)}_k\right)\Big],
\end{split}
\end{equation}
and, collecting the treatment contribution as
$\psi_k(a)=\kappa\left(-a\,\theta(\bar{H}_k) - 0.08\,A_{k-1} + 0.10\,a\,A_{k-1}\right)$,
the event is drawn from the one-step hazard
\begin{align}
\lambda_k(a) &= \mathbb{P}\!\left(T \in [k, k+1) \mid \bar{H}_k,\, A_k = a,\, E_k = 1\right) \nonumber \\
&= \mathrm{expit}\Big(\beta_0 + 0.035\,k + 0.58\,S_k - 0.30\,R_k + 0.16\,B_k
   + 0.22\,\bar S^{(4)}_k + 0.10\,\bar B^{(8)}_k \nonumber \\
&\qquad\qquad + 0.12\,\Delta S^{(3)}_k + 0.12\,W_1 + \psi_k(a)\Big).
\end{align}
Since $\theta$ is positive almost everywhere, treatment is protective and the size of its
benefit is modified by the history. Censoring is resolved on the successor history, so that
dropout depends on covariates realized after the landmark,
\begin{equation}
\begin{split}
G_j(\bar{H}_j) &= \mathbb{P}(C > j \mid \bar{H}_j,\, T > j,\, C > j-1) \\
  &= \mathrm{expit}\Big(c_0 - 0.28\,S_j + 0.16\,R_j - 0.10\,B_j
  - 0.12\,\bar S^{(4)}_j - 0.08\,A_{j-1} - 0.020\,j\Big).
\end{split}
\end{equation}
Censoring is therefore informative given $\bar{H}_t$ while remaining conditionally independent of the potential event time
given $\bar{H}_j$, as assumption (A4) requires. Table~\ref{tab:dgp-synthetic} lists the
parameters.

\begin{table}[ht]
\centering
\caption{Parameters of the synthetic data-generating process.}
\label{tab:dgp-synthetic}
\begin{tabular}{llr}
\toprule
Symbol & Description & Value \\
\midrule
$K$        & decision times                          & 32 \\
$\alpha$   & overlap strength                        & 0.75 \\
$\gamma$   & effect heterogeneity strength           & 3.0 \\
$\phi$     & treatment--confounder feedback strength & 0.8 \\
$\beta_0$  & event intercept                         & $-4.0$ \\
$c_0$      & censoring intercept                     & 5.2 \\
$\kappa$   & treatment effect scale                  & 1.0 \\
$c_\pi$    & propensity clip                         & 0.08 \\
\bottomrule
\end{tabular}
\end{table}

\subsection{Experiment}
\label{app:synthetic-experiment}

\textbf{Design.} Each repetition simulates three independent cohorts of 1{,}800 nuisance,
1{,}800 second-stage, and 1{,}200 test patients, so no patient contributes to more than one
cohort, and we repeat the whole experiment five times with independent seeds. We predict at
the landmark $t=10$ and compare treatment at each of the first $L$ decisions with no
treatment at any of them, both followed by the natural course, for
$(L,\tau)\in\{(2,2),(2,3),(2,4),(3,8)\}$. Among test patients at risk at the landmark,
50.1\% are treated at the landmark, 3.9\% and 14.6\% have an event within two and eight
steps, and 1.7\% and 5.3\% are censored within the same windows. The share of patients
whose recorded treatment matches the all-ones sequence over the intervention window is
27.2\% at $L=2$ and 11.5\% at $L=3$, and 25.9\% and 13.3\% for the all-zeros sequence, so
support is markedly scarcer at the longest setting.

\textbf{True effects.} For each test patient we simulate 1{,}000 futures under each
sequence from the patient's landmark history, with the same random numbers for both
sequences and without censoring, and take the mean difference in survival as the true
effect. Its standard deviation across patients is 0.037--0.042 depending on the setting.

\textbf{Learners and training.} The nuisance cohort fits the treatment and censoring models,
the regime-specific continuation regressions during the intervention window, and a shared
post-intervention continuation. The second-stage cohort fits the second-stage regressions of
RA, IPW, and DR, and the TV-SurvCaus model. Training pools all landmarks at which the
horizon fits within follow-up. All models use the same causal Transformer: one input
projection over the concatenated history, fixed sinusoidal positions, a single causal
encoder layer, and a two-layer ReLU head, with $W$ repeated across time and entering the
same projection. The continuation and second-stage models use width 30, three heads,
feed-forward width 20, and dropout 0.1; the treatment and censoring models use width 40,
four heads, and no dropout. We train with Adam at learning rate $10^{-3}$ and batch size
64, for at most 30 epochs with early stopping on a 15\% patient-level validation split of
each cohort, restoring the checkpoint with the lowest validation loss, and we never use the
test cohort for any selection. The $L_2$ penalty is 0 for the nuisances and 0.01 for the
second stages. Before inversion, treatment and censoring-continuation probabilities are
floored at 0.03, and cumulative inverse-probability products are truncated at 20 for
numerical stability.
TV-SurvCaus \citep{abraich2025tv} is our adaptation of the method to the same Transformer
and to finite interventions.

\textbf{Results.} The full RMSE comparison is given in Table~\ref{tab:synthetic-main};
Table~\ref{tab:synthetic} reports the paired comparison. DR has the lowest mean RMSE in every
setting, and the 95\% interval of its paired difference with the best competing learner lies
below zero in every case except $(L,\tau)=(3,8)$, where the best competitor is PI and the
interval includes zero. The PI, RA, and TV-SurvCaus predictions have a standard deviation
across patients of at most 0.025, against 0.037--0.042 for the true effect, and they
underestimate the average effect by 0.025--0.046; TV-SurvCaus is the most extreme, with a
spread of 0.001--0.005. Their correlation with the true effect is at most 0.25, and is
negative for PI in the three shorter settings, whereas DR attains 0.38--0.74 with a mean
absolute bias of at most 0.016. IPW is nearly unbiased for the average effect in the
shorter settings but is heavily over-dispersed, with a prediction spread of 0.097--0.261
against the same true spread, which is why its RMSE is the largest throughout; at $(3,8)$
its effective sample size falls to 116 of 1{,}200 eligible patients. Refitting the DR and
IPW second stages with the true
treatment and censoring probabilities of the simulation, keeping the fitted continuation
regressions and other stabilization choices fixed, changes the mean DR error by at most 0.008 with a
paired interval that includes zero in all four settings. The same substitution lowers the
mean IPW error by 0.010--0.057, although its paired interval lies below zero only at
$(3,8)$.

\begin{table}[ht]
\centering
\small
\setlength{\tabcolsep}{4pt}
\begin{tabular}{lccc}
\toprule
$(L,\tau)$ & Best competitor & DR $-$ best competitor & Paired 95\% interval \\
\midrule
$(2,2)$ & TV-SurvCaus & $-0.0272$ & $[-0.0360, -0.0183]$ \\
$(2,3)$ & PI & $-0.0167$ & $[-0.0323, -0.0011]$ \\
$(2,4)$ & PI & $-0.0157$ & $[-0.0302, -0.0012]$ \\
$(3,8)$ & PI & $-0.0042$ & $[-0.0181, 0.0097]$ \\
\bottomrule
\end{tabular}
\caption{Paired difference in RMSE between DR and the competing learner with the lowest mean on the fully synthetic benchmark, with its paired 95\% $t$-interval over five repetitions; a negative value favors DR.}
\label{tab:synthetic}
\end{table}

\section{Semi-synthetic experiment}
\label{app:semisynthetic}

The benchmark keeps each patient's recorded ICU history up to the landmark and simulates everything after it. As the history is real, the learners must work with heterogeneous and partly missing clinical measurements, and as the future is simulated, the true effect is known for every test patient.

\textbf{Cohort.} We use MIMIC-IV version 3.1 \citep{johnson2023mimic} and place the landmark 24 hours after ICU admission. An ICU stay is eligible if the patient is at least 18 years old at admission, is still in the ICU at the landmark, and has not died in hospital before it. We compute age as the anchor age plus the years between the anchor year and the admission year, and we keep the earliest eligible stay of each patient, which leaves 54{,}375 patients. We select 6{,}000 of them with a hash of the patient identifier, and a second hash assigns 40\%, 40\%, and 20\% of the selected patients to the nuisance, second-stage, and test cohorts, which hold 2{,}380, 2{,}422, and 1{,}198 patients. Neither hash depends on any measurement or outcome. The median age is 66 years (interquartile range 55--77), and 43.2\% of the patients are female. We repeat the study five times on the same patients with new simulation seeds and report the RMSE of $\widehat\theta$ over the test patients.

\textbf{History before the landmark.} We read heart rate (item 220045), mean arterial pressure (items 220052 and 220181, from arterial and non-invasive measurement), and respiratory rate (item 220210) from the \texttt{chartevents} table. We pick the value at the later of its chart time and its store time. We drop records that lack either time, that were charted before ICU admission, or that became available after ICU discharge, and we keep values in the expected unit within plausible ranges, which are 20--250 beats per minute for heart rate, 15--250 mmHg for mean arterial pressure, and 1--80 breaths per minute for respiratory rate. We split the first 24 hours into eight bins of three hours and place each record in the bin that contains its availability time, so a value enters a bin only if it was available before the bin ends. The value of a bin is the mean of its records. Depending on the variable, 93--95\% of bins hold a measurement, and fewer than 0.6\% of patients have no measurement of a variable in the 24 hours.

We fit all preprocessing on the 2{,}023 nuisance-training patients remaining after the 15\% validation split, and never on second-stage or test patients. A missing first bin takes the training median, and a later missing bin takes the last observed value. We standardize each variable with its training mean and standard deviation over all bins and clip at $\pm5$. The static covariates are standardized age, a female indicator, and the standardized mean of the patient's respiratory-rate history.

\textbf{Simulated follow-up.} The eight real bins become the first eight states of the synthetic data-generating process, so the landmark is the decision time $t=7$. We map heart rate, mean arterial pressure, and respiratory rate to the first, second, and third state coordinates, and age, sex, and mean respiratory rate to the first, second, and third baseline covariates. We set every treatment before the landmark to zero, as the simulated treatment has no counterpart in the records. From the landmark on, the structural equations and parameters of the synthetic study generate treatment $A_k$, the event in $[k,k+1)$, the next state $X_{k+1}$, and censoring at $k+1$, for $k=7,\ldots,31$. At the landmark these equations read the real history through its last three states, its four- and eight-step means, and its three-step trend, so patients with different recorded histories have different treatment probabilities, event risks, and treatment effects. Censoring depends on the next state, its recent mean, and the last treatment, which makes it informative given the landmark history but ignorable given the full history, as (\ref{ass:sequential-ignorability}) requires. Averaged over the five repetitions, 52\% of patients in the nuisance and second-stage cohorts receive treatment at the landmark, 2.8\% and 12.8\% have an event within two and eight steps, and 4.9\% are censored within eight steps.

\textbf{True effects.} At the landmark we compare treatment at each of the first $L$ decisions with no treatment at any of them, both followed by the natural course, and evaluate survival $\tau$ steps after the landmark, for $(L,\tau)\in\{(2,2),(2,3),(2,4),(3,8)\}$. For each test patient we simulate 1{,}000 futures under each sequence from the patient's 24-hour history, with the same random numbers for both sequences and without censoring, and take the mean difference in survival as the true effect.

\textbf{Learners and training.} The nuisance cohort fits the treatment and censoring models, the regime-specific continuation regressions during the intervention window, and a shared post-intervention continuation. The second-stage cohort fits the second-stage regressions of RA, IPW, and DR, and the TV-SurvCaus model. Training pools all landmarks from $t=7$ on. Each test patient contributes only the 24-hour history, with all later values removed before prediction. All learners use the same causal Transformer, optimizer, and early-stopping setup as in the fully synthetic experiment (App.~\ref{app:synthetic-experiment}). Before inversion, treatment and censoring-continuation probabilities are floored at 0.03, and cumulative inverse-probability products are truncated at 10 for numerical stability. We use the same TV-SurvCaus adaptation \citep{abraich2025tv}.

\textbf{Results.} The full RMSE comparison is given in Table~\ref{tab:semisynthetic-main}; Table~\ref{tab:semisynthetic} reports the paired comparison. DR has the lowest mean RMSE in every setting, and the 95\% interval of its paired difference with the best competing learner lies below zero in all four settings. The PI, RA, and TV-SurvCaus predictions have a standard deviation across patients of at most 0.011, against 0.050--0.059 for the true effect, and they underestimate the average effect by 0.035--0.051. The DR predictions have a correlation of 0.36--0.64 with the true effect, against at most 0.30 for PI, RA, and TV-SurvCaus, and a mean absolute bias of at most 0.024.

\begin{table}[ht]
\centering
\small
\setlength{\tabcolsep}{4pt}
\begin{tabular}{lccc}
\toprule
$(L,\tau)$ & Best competitor & DR $-$ best competitor & Paired 95\% interval \\
\midrule
$(2,2)$ & TV-SurvCaus & $-0.0273$ & $[-0.0357, -0.0190]$ \\
$(2,3)$ & RA & $-0.0199$ & $[-0.0342, -0.0055]$ \\
$(2,4)$ & PI & $-0.0159$ & $[-0.0278, -0.0040]$ \\
$(3,8)$ & PI & $-0.0160$ & $[-0.0315, -0.0004]$ \\
\bottomrule
\end{tabular}
\caption{Paired difference in RMSE between DR and the competing learner with the lowest mean on the semi-synthetic MIMIC-IV benchmark, with its paired 95\% $t$-interval over five repetitions; a negative value favors DR.}
\label{tab:semisynthetic}
\end{table}

\subsection{Nuisance-misspecification experiment}
\label{app:semisynthetic-misspecification}

To see whether the advantage of DR comes from its double robustness, we misspecify one nuisance block at a time on the semi-synthetic benchmark. When we replace the treatment and censoring probabilities with their pooled rates, which ignore the history, the error of IPW grows between five- and twelvefold, while the error of DR does not increase for $L=2$ and doubles for $L=3$ and $\tau=8$. When we instead lower the survival predicted by the continuation regressions of the treated arm by 0.10, the error of RA doubles and the error of DR does not increase, with either the fitted or the true probabilities. When both nuisance blocks are misspecified, the error of DR increases in every setting. Most of the gain of DR over PI and RA comes from removing their average bias, which accounts for 59--76\% of the difference in mean squared error.

\subsection{Censoring and treatment-assignment stress test}
\label{app:semisynthetic-stress}

We stress-test the learners along three dimensions that make longitudinal survival-effect estimation increasingly challenging. We use the same patients and preprocessing as in the main semi-synthetic experiment, draw new simulation seeds, and evaluate $(L,\tau)=(2,2)$ and $(3,8)$. For this stress test, cumulative inverse-probability products are truncated at 20.

\begin{enumerate}
    \item \textbf{History-dependent censoring.}
    We vary the amount of censoring to assess robustness to increasingly severe selective loss to follow-up. For both treatment-assignment mechanisms below, we consider settings in which approximately 0\%, 20\%, or 40\% of patients are censored before the eight-step horizon.

    \item \textbf{Treatment-assignment strength.}
    We vary how strongly treatment depends on the evolving patient history. The default assignment uses the treatment logit of the synthetic study scaled by $0.75$, whereas the \emph{sharp} assignment scales it by $2.0$. The latter induces stronger treatment selection and more challenging overlap while keeping treatment probabilities within $[0.08,0.92]$.

    \item \textbf{Event frequency.}
    We additionally consider a setting with more frequent survival events, obtained by increasing the event intercept from $-4$ to $-2$, together with 20\% censoring. This tests performance when substantially fewer patients remain event-free over follow-up.
\end{enumerate}

Crossing the two treatment-assignment mechanisms with the three censoring levels gives six conditions, and the higher-event setting gives a seventh. We calibrate the censoring intercept by bisection on uncensored simulated futures of 1{,}024 nuisance patients, yielding 3.65 for 20\% censoring, 2.78 for 40\%, and 3.24 in the higher-event condition. In the nuisance and second-stage cohorts, the resulting censoring rates before the eight-step horizon are 19.6\% and 39.4\% under the default assignment, 19.9\% and 40.1\% under the sharp assignment, and 20.4\% in the higher-event condition. Censoring does not alter the true effect. In contrast, the sharp assignment changes the natural-course treatment distribution after the intervention and therefore the true effect at $(3,8)$, while the higher event rate changes the true effect at both horizons.

Table~\ref{tab:semisynthetic-stress} reports the results across these stress conditions. At the short horizon $(2,2)$, DR has the lowest RMSE in all seven settings, including under increased censoring, sharper treatment assignment, and the higher event rate; the 95\% intervals of its paired differences with PI, RA, and the best competing learner all lie below zero. At the longer horizon $(3,8)$, DR has the lowest mean RMSE in five of the seven settings, although its paired interval against the best competing learner lies below zero only under the default assignment with 40\% censoring. RA has the lowest mean in the remaining two settings---default assignment with 20\% censoring and sharp assignment with 40\% censoring---but in both cases its difference from DR is not statistically resolved by the paired interval. IPW has the largest error at $(3,8)$ in every condition and deteriorates particularly under sharp treatment assignment.

\begin{table}[ht]
\centering
\small
\setlength{\tabcolsep}{3pt}
\begin{tabular}{llccccccl}
\toprule
Setting & Cens. & PI & RA & IPW & TV-SurvCaus & DR & Rel. improv. & DR $-$ best competitor \\
\midrule

\multicolumn{9}{l}{\textit{Short horizon: $(L,\tau)=(2,2)$}} \\
\addlinespace[1pt]
Default & 0\%  & 0.069 & 0.071 & 0.141 & 0.073 & \textbf{0.041} & \textbf{+40.6\%} & $-0.0273$\,{\scriptsize$[-0.0403, -0.0143]$} {\scriptsize(PI)} \\
        & 20\% & 0.071 & 0.074 & 0.131 & 0.073 & \textbf{0.042} & \textbf{+40.8\%} & $-0.0292$\,{\scriptsize$[-0.0431, -0.0153]$} {\scriptsize(PI)} \\
        & 40\% & 0.073 & 0.075 & 0.178 & 0.073 & \textbf{0.041} & \textbf{+43.8\%} & $-0.0318$\,{\scriptsize$[-0.0357, -0.0278]$} {\scriptsize(TV)} \\
\addlinespace[2pt]

Sharp   & 0\%  & 0.074 & 0.078 & 0.293 & 0.073 & \textbf{0.044} & \textbf{+39.7\%} & $-0.0291$\,{\scriptsize$[-0.0368, -0.0213]$} {\scriptsize(TV)} \\
        & 20\% & 0.071 & 0.074 & 0.311 & 0.073 & \textbf{0.041} & \textbf{+42.3\%} & $-0.0296$\,{\scriptsize$[-0.0331, -0.0262]$} {\scriptsize(PI)} \\
        & 40\% & 0.073 & 0.075 & 0.353 & 0.073 & \textbf{0.051} & \textbf{+30.1\%} & $-0.0215$\,{\scriptsize$[-0.0323, -0.0107]$} {\scriptsize(TV)} \\
\addlinespace[2pt]

Higher event & 20\% & 0.141 & 0.149 & 0.155 & 0.257 &
\textbf{0.085} & \textbf{+39.7\%} & $-0.0560$\,{\scriptsize$[-0.0774, -0.0345]$} {\scriptsize(PI)} \\

\midrule

\multicolumn{9}{l}{\textit{Long horizon: $(L,\tau)=(3,8)$}} \\
\addlinespace[1pt]
Default & 0\%  & 0.071 & 0.068 & 0.184 & 0.075 & \textbf{0.056} & \textbf{+17.6\%} & $-0.0115$\,{\scriptsize$[-0.0339, 0.0108]$} {\scriptsize(RA)} \\
        & 20\% & 0.068 & \textbf{0.066} & 0.235 & 0.074 & 0.078 & $-18.2\%$ & $+0.0120$\,{\scriptsize$[-0.0126, 0.0366]$} {\scriptsize(RA)} \\
        & 40\% & 0.068 & 0.068 & 0.211 & 0.075 & \textbf{0.049} & \textbf{+27.9\%} & $-0.0186$\,{\scriptsize$[-0.0296, -0.0076]$} {\scriptsize(RA)} \\
\addlinespace[2pt]

Sharp   & 0\%  & 0.070 & 0.066 & 0.477 & 0.071 & \textbf{0.057} & \textbf{+13.6\%} & $-0.0086$\,{\scriptsize$[-0.0204, 0.0031]$} {\scriptsize(RA)} \\
        & 20\% & 0.070 & 0.067 & 0.495 & 0.073 & \textbf{0.060} & \textbf{+10.4\%} & $-0.0068$\,{\scriptsize$[-0.0248, 0.0113]$} {\scriptsize(RA)} \\
        & 40\% & 0.072 & \textbf{0.070} & 0.564 & 0.071 & 0.076 & $-8.6\%$ & $+0.0053$\,{\scriptsize$[-0.0128, 0.0234]$} {\scriptsize(RA)} \\
\addlinespace[2pt]

Higher event & 20\% & 0.088 & 0.082 & 0.160 & 0.119 &
\textbf{0.077} & \textbf{+6.1\%} & $-0.0048$\,{\scriptsize$[-0.0387, 0.0291]$} {\scriptsize(RA)} \\

\bottomrule
\end{tabular}

\caption{RMSE of the estimated HSTE under the stress-test conditions,
averaged over five repetitions (lower is better; lowest mean
in bold). \emph{Default} and \emph{Sharp} denote the standard and stronger
history-dependent treatment-assignment mechanisms, respectively. Cens.
denotes the target proportion censored before the eight-step horizon;
\emph{Higher event} denotes the increased-event-rate condition.
Rel. improv. is the percentage reduction in RMSE of DR relative to the
best competing learner; positive values favor DR. The last column gives the paired difference in RMSE between DR and the competing learner with the lowest mean (named in parentheses, with TV for TV-SurvCaus), with its 95\% $t$-interval over the five repetitions; a negative value favors DR.}
\label{tab:semisynthetic-stress}
\end{table}

\section{Real-world experiment}
\label{app:realworld}

As real data contain no counterfactual outcomes, we cannot evaluate an estimated effect against counterfactual ground truth. We can still check whether the survival a learner predicts under one treatment sequence agrees with what happened to the patients who received that sequence. \citet{frauen2025modelagnostic} compare predictions with the outcomes of such patients directly, whereas we follow \citet{keogh2024prediction} and \citet{boyer2025estimating} and also weight each patient by the inverse probability of following the sequence, which makes the score refer to all patients at the landmark and not only to those who followed it. MIMIC-IV also records deaths after a patient leaves the ICU, and we therefore know the survival at each horizon even for patients whose ICU record ends in censoring.

\textbf{Cohort and sequence.} We take the 8{,}243 adults in MIMIC-IV version 3.1 \citep{johnson2023mimic} who were still in the ICU 24 hours after admission and had a vasopressor (norepinephrine, epinephrine, vasopressin, phenylephrine, or dopamine) running in the three hours before. We split time into three-hour steps from ICU admission. The history at each step contains heart rate, systolic, diastolic, and mean arterial pressure, respiratory rate, oxygen saturation, temperature, the norepinephrine rate, and the number of vasopressors running, together with age and sex, and we use each measurement only from the time it was charted. A patient receives treatment at a step if any vasopressor infusion runs during it. Death is the event and leaving the ICU alive censors the record, and as patients leave the ICU when their condition improves, the censoring depends on the history. At the 24-hour landmark we predict survival under continuing vasopressors for the next two steps ($L=2$) and evaluate it 24 and 48 hours later ($\tau=8,16$). We assign 40\%, 40\%, and 20\% of the patients to the nuisance, second-stage, and test cohorts using a fixed hash of the patient identifier. The learners, their training, and the stabilization choices are the same as in the semi-synthetic study, and we fit each learner with five seeds.

\textbf{Evaluation.} In the held-out test cohort, 1{,}285 of the 1{,}655 patients have a known outcome and continued vasopressors through both steps or until death, and 140 of them died within 48 hours. As the decision to continue depends on the history, we weight each of these patients by the inverse of their estimated probability of continuing, which we obtain from logistic regressions fitted on the test cohort with two-fold cross-fitting \citep{keogh2024prediction}. Patients who stop contribute nothing, and because every stop happens before the horizon, dropping them is the same as censoring them when they stop. The weighted Brier score then measures the error of the predicted survival under continued vasopressors for all test patients at the landmark. The weights stay close to one, and none exceeds 4.4. Of the 1{,}285 patients, 224 left the ICU alive within 48 hours, and the death records show that 221 of them were still alive at 48 hours. We cannot score the sequence that stops vasopressors in the same way, because only 23 patients in the whole cohort stopped and died within 48 hours.

\textbf{Results.} Table~\ref{tab:realworld-main} gives the full weighted Brier comparison on the held-out test cohort, and Table~\ref{tab:realworld-app} gives the paired comparison. At 48 hours DR has the lowest weighted Brier score and the paired interval of its difference with the best competing learner lies below zero, while at 24 hours DR and PI are tied and both score better than RA, IPW, and TV-SurvCaus. The mean prediction of DR is within 2.6 percentage points of the weighted observed survival on the held-out test cohort, whereas PI and RA overpredict survival at 48 hours, by 5.6 and 6.7 points. PI and RA rely only on continuation regressions that we fit backwards one step at a time, and small upward errors at each step add up over a long horizon. The weighted correction in DR removes more than half of the excess. The check concerns the predicted survival under one sequence and not the effect itself, and it also relies on leaving the ICU being explained by the recorded history \citep{vangeloven2020prediction}, which would not hold for a transfer out of the ICU for end-of-life care.

\begin{table}[ht]
\centering
\small
\setlength{\tabcolsep}{4pt}
\begin{tabular}{lccc}
\toprule
Horizon & Best competitor & DR $-$ best competitor & Paired 95\% interval \\
\midrule
24 hours ($\tau=8$) & PI & $-0.0001$ & $[-0.0013, 0.0012]$ \\
48 hours ($\tau=16$) & IPW & $-0.0038$ & $[-0.0059, -0.0018]$ \\
\bottomrule
\end{tabular}
\caption{Paired difference in weighted Brier score between DR and the competing learner with the lowest mean under continued vasopressors for the next two steps ($L=2$) in the held-out MIMIC-IV test cohort. The 95\% bootstrap intervals over test patients are computed on seed-averaged predictions with the weights held fixed; a negative value favors DR.}
\label{tab:realworld-app}
\end{table}

\section{Computational resources}
\label{app:computational-resources}

The fully synthetic experiments were run on an NVIDIA GeForce RTX 4060 GPU.
The semi-synthetic experiments, including the nuisance-misspecification and
stress-test studies, were run on a single NVIDIA H200 GPU per run, requiring
approximately 5 minutes per run. The real-world experiments were run on a
single NVIDIA RTX PRO 5000 Blackwell GPU with 48\,GB of memory, requiring
approximately 2--4 minutes per seed.

\end{document}

%% file: math_commands.tex
\usepackage{amsmath,amsfonts,bm}

\def\eqref#1{equation~\ref{#1}}

\def\1{\bm{1}}

\def\eps{{\epsilon}}

\DeclareMathAlphabet{\mathsfit}{\encodingdefault}{\sfdefault}{m}{sl}
\SetMathAlphabet{\mathsfit}{bold}{\encodingdefault}{\sfdefault}{bx}{n}

